\documentclass{article}

\usepackage[nonatbib,preprint]{neurips_2021}

\usepackage{xr}
\makeatletter
\newcommand*{\addFileDependency}[1]{
  \typeout{(#1)}
  \@addtofilelist{#1}
  \IfFileExists{#1}{}{\typeout{No file #1.}}
}
\makeatother

\usepackage[utf8]{inputenc} 
\usepackage[T1]{fontenc}    
\usepackage{hyperref}       
\usepackage{url}            
\usepackage{booktabs}       
\usepackage{amsfonts}       
\usepackage{nicefrac}       
\usepackage{microtype}      
\usepackage[dvipsnames]{xcolor}

\usepackage{multirow}
\usepackage{microtype}
\usepackage{graphicx}
\usepackage{lscape}
\usepackage{amsmath}
\usepackage{amsthm}
\usepackage{mathtools}
\usepackage{amsfonts}
\usepackage{bm}
\usepackage{wrapfig}
\usepackage{placeins}
\usepackage{comment}
\usepackage{ulem}
\usepackage{hyperref}
\usepackage{multicol}
\usepackage{enumitem}
\usepackage{authblk}
\usepackage{subcaption} 
\usepackage{graphicx} 

\newcommand{\bA}{\mathbf{A}}
\newcommand{\bD}{\mathbf{D}}
\newcommand{\bH}{\mathbf{H}}
\newcommand{\bI}{\mathbf{I}}
\newcommand{\bL}{\mathbf{L}}

\newcommand{\bU}{\mathbf{U}}

\newcommand{\bPi}{\bm{\Pi}}

\newcommand{\bTheta}{\bm{\Theta}}

\usepackage{array}
\newcolumntype{H}{>{\setbox0=\hbox\bgroup}c<{\egroup}@{}}

\newtheorem{example}{Example}

\newtheorem{theorem}{Theorem}

\title{MagNet: A 
Neural Network for Directed Graphs}

\author[1]{Xitong Zhang}
\author[2]{Yixuan He}
\author[1,3]{Nathan Brugnone}
\author[4]{Michael Perlmutter}
\author[1,5,6]{Matthew Hirn}
\affil[1]{Michigan State University, Department of Computational Mathematics, Science \& Engineering\\ East Lansing, Michigan, United States}
\affil[2]{University of Oxford, Department of Statistics, Oxford, England, United Kingdom}
\affil[3]{Michigan State University, Department of Community Sustainability,\newline East Lansing, Michigan, United States}
\affil[4]{University of California, Los Angeles, Department of Mathematics,\newline Los Angeles, California, United States}
\affil[5]{Michigan State University, Department of Mathematics, \newline East Lansing, Michigan, United States}
\affil[6]{Michigan State University, Center for Quantum Computing, Science \& Engineering,\newline East Lansing, Michigan, United States}
\begin{document}

\maketitle

\begin{abstract}
The prevalence of graph-based data has spurred the rapid development of graph neural networks (GNNs) and related machine learning algorithms. Yet, despite the many datasets naturally modeled as directed graphs, including citation, website, and traffic networks, the vast majority of this research focuses  
on undirected graphs. In this paper, we propose \textit{MagNet}, a spectral GNN for directed graphs based on a complex Hermitian matrix known as the magnetic Laplacian. This matrix encodes undirected geometric structure in the magnitude of its entries and directional information in their phase. A ``charge'' parameter attunes spectral information to variation among directed cycles. We apply our network to a variety of directed graph node classification and link prediction tasks showing that MagNet  performs well on all tasks and that its performance exceeds all other methods on a  majority of such tasks.  
The underlying principles of MagNet are such that it can be 
adapted to 
other spectral GNN architectures.
\end{abstract}

\section{Introduction}\label{sec: intro}

Endowing a collection of objects with a graph structure allows one to encode pairwise relationships among its elements. These relations often possess a natural notion of direction. For example, the WebKB dataset        \cite{pei2020geom} 
contains a list of university websites with associated hyperlinks. In this context, one website might link to a second without a reciprocal link to the first. Such datasets are naturally modeled by \textit{directed graphs}. In this paper, we introduce \textit{MagNet}, a graph convolutional neural network for directed graphs based on the magnetic Laplacian.

Most graph neural networks fall into one of two families, \textit{spectral networks} or \textit{spatial networks}.  Spatial methods define graph convolution as a localized averaging operation with iteratively learned weights. Spectral networks, on the other hand, define convolution on graphs via the eigendecompositon of the (normalized) graph Laplacian. The eigenvectors of the graph Laplacian assume the role of Fourier modes, and convolution is defined as entrywise multiplication in the Fourier basis.   
For a comprehensive review of both spatial and spectral networks, we refer the reader to \cite{zhou2018graph} and \cite{wu:gnnSurvey2020}.

Many spatial graph CNNs have natural extensions to directed graphs. However, it is common for these methods to preprocess the data by symmetrizing the adjacency matrix, effectively creating an undirected graph. For example, while \cite{velivckovic2017graph} explicitly notes that their network is well-defined on directed graphs, their experiments treat all citation networks as undirected for improved performance. 

Extending spectral methods to directed graphs is not straightforward since the adjacency matrix is asymmetric and, thus, there is no obvious way to define a symmetric, real-valued Laplacian with a full set of real eigenvalues that uniquely encodes any directed graph. We overcome this challenge by constructing a network based on the  magnetic Laplacian $\mathbf{L}^{(q)}$ defined in Section \ref{sec: mag}. Unlike the directed graph Laplacians used in works such as \cite{ma:spectralDGCN2019, monti:MotifNet2018, Tong2020DigraphIC,tong:directedGCN2020}, the magnetic Laplacian is not a real-valued symmetric matrix. Instead, it is a \textit{complex-valued Hermitian} matrix that encodes the fundamentally asymmetric nature of a directed graph via the  
complex phase of its entries. 

Since $\bL^{(q)}$ is Hermitian, the spectral theorem implies it has an orthonormal basis of complex eigenvectors corresponding to real eigenvalues.  
Moreover, Theorem \ref{thm: posdef}, stated in Appendix \ref{sec: eigs}, shows that $\mathbf{L}^{(q)}$ is positive semidefinite, similar to the traditional Laplacian.
Setting $q=0$ is equivalent to symmetrizing the adjacency matrix  and no importance is given to directional information. When $q=.25$, on the other hand, we have that $\bL^{(.25)}(u,v)=-\bL^{(.25)}(v,u)$ whenever there is an edge from $u$ to $v$ but not from $v$ to $u$.  Different values of $q$ highlight different graph motifs \cite{fanuel:magneticEigenmaps2018, fanuel2017Magnetic, krystal:hermitianAdjDigraphs2017, mohar2020new}, and therefore the optimal choice of $q$ varies. 
Learning the appropriate value of $q$ from data allows MagNet to adaptively incorporate directed information. We also note that $\mathbf{L}^{(q)}$ has been applied to graph signal processing  \cite{furutani:GSPdirectedHermit2019}, community detection  \cite{fanuel2017Magnetic}, and clustering \cite{CLONINGER2017370, fanuel:magneticEigenmaps2018, f2020characterization}.

In Section \ref{sec: conv}, we show how the networks constructed in \cite{bruna:spectralNN2014, Defferrard2018, kipf2016semi} can be adapted to  directed graphs by incorporating complex Hermitian matrices, such as the magnetic Laplacian. When $q=0$, we effectively recover the networks constructed in those previous works. Therefore, our work  generalizes these  networks in a way that is suitable for directed graphs. Our method is very general and is not tied to any
particular choice of  network architecture. Indeed, the main ideas of this work could be adapted to nearly any spectral network.

In Section \ref{sec: related}, we summarize related work on directed graph neural networks  as well as other papers studying the magnetic Laplacian and its applications in data science. In Section \ref{sec: experiments}, we apply our network to node classification and link prediction tasks. We compare against several spectral and spatial methods as well as networks designed for directed graphs. We find that MagNet obtains the best or second-best performance on five out of six node-classification tasks and has the best performance on seven out of eight link-prediction tasks tested on real-world data, in addition to providing excellent node-classification performance on difficult synthetic data. In the appendix, we provide full implementation details, theoretical results concerning the magnetic Laplacian, extended examples, and further numerical details.

\section{The magnetic Laplacian}
\label{sec: mag}

Spectral graph theory has been remarkably successful in relating geometric characteristics of undirected graphs to properties of eigenvectors and eigenvalues of 
graph Laplacians and related matrices. For example, the tasks of optimal graph partitioning, sparsification, clustering, and embedding may be approximated 
by eigenvectors corresponding to small eigenvalues of various  Laplacians (see, e.g., \cite{chung1997spectral, shi1997normalized, belkin2003laplacian, spielman2004nearly, coifman2006diffusion}). Similarly, the graph signal processing research community leverages the full set of eigenvectors to extend the Fourier transform to these structures  \cite{ortega2018graph}. 
Furthermore, numerous papers \cite{bruna:spectralNN2014, Defferrard2018, kipf2016semi} have shown that this eigendecomposition can be used to define  neural networks on graphs.  In this section, we provide the background 
needed to extend these constructions to directed graphs via complex Hermitian matrices such as the magnetic Laplacian. 

We let $G = (V, E)$ be a directed graph where $V$ is a set of $N$ vertices and $E\subseteq V\times  V$ is a set of directed edges. If $(u,v)\in E$, then we say there is an edge from $u$ to $v$. For the sake of simplicity, we will focus on the case where the graph is unweighted and has no self-loops, i.e., $(v,v) \notin E$, but our methods have natural extensions to graphs with self-loops and/or  weighted edges. If both $(u,v) \in E$ and $(v,u) \in E$, then one may consider this pair of directed edges as a single undirected edge. 

A directed graph can be described by an adjacency matrix $(\bA(u,v))_{u,v\in V}$ where $\bA(u,v)=1$  if $(u,v)\in E$ and $\bA(u,v)=0$ otherwise. 
Unless $G$ is undirected, $\bA$ is not symmetric, and, indeed, this is the key technical challenge in extending spectral graph neural networks to directed graphs. In the undirected case, where the adjacency matrix $\mathbf{A}$  is symmetric, the (unnormalized) graph Laplacian can be defined by $\mathbf{L}=\mathbf{D}-\mathbf{A},$ where $\mathbf{D}$ is a diagonal degree matrix. It is well-known that $\mathbf{L}$ is a symmetric, positive-semidefinite matrix and therefore has an orthonormal basis of eigenvectors associated with non-negative eigenvalues. However, when $\mathbf{A}$ is asymmetric, direct attempts to define the Laplacian this way typically yield complex eigenvalues. This impedes the straightforward extension of classical methods of spectral graph theory and graph signal processing to directed graphs. 

A key point of this paper is to represent the directed graph through a complex Hermitian matrix $\bm{\mathcal{L}}$ such that: (1) the magnitude of  $\bm{\mathcal{L}}(u,v)$ indicates the presence of an edge, but not its direction; and (2) the phase of  $\bm{\mathcal{L}}(u,v)$ indicates the direction of the edge, or if the edge is undirected. Such matrices have been explored in the directed graph literature (see Section \ref{sec: related}), but not in the context of graph neural networks. They have several advantages over their real-valued matrix counterparts. 
In particular, a single symmetric real-valued matrix will not uniquely represent a directed graph. 
Instead, one must use several 
matrices, as in \cite{tong:directedGCN2020}, but this increases the complexity of the resulting network. Alternatively, one can work with an asymmetric, real-valued matrix, such as the adjacency matrix or the random walk matrix.  
However, the spatial graph filters that result from such matrices are typically limited by the fact that they can only aggregate information from the 
vertices that can be reached in one hop from 
a central vertex, but 
ignore the equally important subset of vertices 
that can reach the central vertex in one hop. Complex Hermitian matrices, however, 
lead to filters that aggregate information from both sets of vertices.  
 Finally, one could use a real-valued skew-symmetric matrix 
 but such matrices do not generalize well to graphs with both directed and undirected edges.

The optimal choice of  complex Hermitian matrix 
is an open question. Here, we utilize a 
parameterized family of magnetic Laplacians, which have proven to be useful in other data-driven  
contexts \cite{fanuel2017Magnetic, CLONINGER2017370, fanuel:magneticEigenmaps2018, f2020characterization}.
We first define the symmetrized adjacency matrix and corresponding degree matrix by, 
\begin{equation*}
    \bA_s (u,v) \coloneqq \frac{1}{2} ( \bA (u,v) + \bA(v,u) ), \,\:\:1\leq u,v\leq N ,\quad     \bD_s (u,u) \coloneqq \sum_{v \in V} \bA_s (u,v),  \:\:1\leq u\leq N \, ,
\end{equation*}
with $\bD_s(u,v)=0$ for $u\neq v$. We capture  directional information via a phase matrix,\footnote{Our definition of $\bTheta^{(q)}$ coincides with that used in \cite{furutani:GSPdirectedHermit2019}. However, another definition (differing by a minus sign) also appears in the literature. These resulting magnetic Laplacians have the same eigenvalues and the corresponding eigenvectors are complex conjugates of one another. Therefore, this difference does not affect the performance of our network since our final layer separates the real and imaginary parts before multiplying by a trainable weight matrix (see Section \ref{sec: conv} for details on the network structure).} $\bTheta^{(q)}$, 
\begin{equation*}
    \bTheta^{(q)} (u,v) \coloneqq 2\pi q (\bA(u,v) - \bA(v,u)) \,, \quad q\geq 0 \, ,
\end{equation*}
where $\exp (i \bTheta^{(q)})$ is defined component-wise by
  $  \exp (i \bTheta^{(q)}) (u,v) \coloneqq \exp(i \bTheta^{(q)} (u,v))$. 
Letting $\odot$ denotes componentwise multiplication, we define the complex Hermitian adjacency matrix $\bH^{(q)}$ by
\begin{equation*}
    \bH^{(q)} \coloneqq \bA_s \odot \exp (i \bTheta^{(q)}) \, .
\end{equation*}
Since  $\bTheta^{(q)}$ is skew-symmetric,  $\bH^{(q)}$ is Hermitian. 
When $q=0$, we have $\bTheta^{(0)} = \bm{0}$ and so $\bH^{(0)}=\bA_s$. 
This  effectively  corresponds to  treating the graph as undirected. For $q \neq 0$, the phase of $\bH^{(q)}$ encodes edge direction. For example, if  there is an edge from $u$ to $v$ but not from $v$ to $u$ we have
\begin{equation*}
    \bH^{(.25)}(u,v)=\frac{i}{2}=-\bH^{(.25)}(v,u) \, .
\end{equation*}
Thus, in this setting, an edge from $u$ to $v$ is treated as the opposite of an edge from $v$ to $u$. On the other hand, if $(u,v), (v,u) \in E$ (which can be interpreted as a single undirected edge), then $\mathbf{H}^{(q)}(u,v) = \mathbf{H}^{(q)}(v,u) = 1$, and we see the phase, $\bm{\Theta}^{(q)}(u,v) = 0$, encodes the lack of direction in the edge. For the rest of this paper, we will assume that $q$ lies in between these two extreme values, i.e., $0\leq q \leq .25.$ We define the normalized and unnormalized  magnetic Laplacians by
\begin{equation}
    \bL_{U}^{(q)} \coloneqq \bD_s - \bH^{(q)} = \bD_s - \bA_s \odot \exp (i \bTheta^{(q)}),\quad \bL_{N}^{(q)} \coloneqq  \bI - \left(\bD_s^{-1/2}\bA_s\bD_s^{-1/2}\right) \odot \exp (i \bTheta^{(q)}) \, .
\end{equation}
Note that when $G$ is undirected, $\bL_U^{(q)}$ and $\bL_N^{(q)}$ reduce to the standard undirected Laplacians.

$\bL_U^{(q)}$ and $\bL_N^{(q)}$ are Hermitian. Theorem 1 (see Appendix \ref{sec: eigs}) shows they are positive-semidefinite and thus are diagonalized by an orthonormal basis of complex eigenvectors $\mathbf{u}_1, \ldots, \mathbf{u}_N$ associated to real, nonnegative eigenvalues $\lambda_1, \ldots, \lambda_N$. 
Similar to the traditional normalized Laplacian, Theorem 2 (see Appendix \ref{sec: eigs}) shows that the eigenvalues of $\mathbf{L}^{q}_N$ lie in  $[0,2]$, and we may factor $\bL_N^{(q)}= \mathbf{U} \bm{\Lambda} \mathbf{U}^\dagger$, 
where $\mathbf{U}$ is the $N\times N$ matrix whose $k$-th column is $\mathbf{u}_k$, $\bm{\Lambda}$ is the diagonal matrix with $\bm{\Lambda}(k,k)=\lambda_k$, and $\mathbf{U}^{\dagger}$ is the conjugate transpose of $\mathbf{U}$ (a similar formula holds for $\bL^{(q)}_U$). The magnetic Laplacian encodes geometric information in its eigenvectors and eigenvalues. In the directed star graph (see Appendix \ref{sec: star}), for example, directional information is contained in the eigenvectors only, whereas the eigenvalues are invariant to the direction of the edges. On the other hand, for the directed cycle graph the magnetic Laplacian encodes the directed nature of the graph solely in its spectrum. In general, both the eigenvectors and eigenvalues may contain important information. In Section \ref{sec: conv}, we will introduce MagNet, a network designed to leverage this spectral information. 
	  
\section{MagNet}
\label{sec: conv}

Most graph neural network architectures can be described as being either \textit{spectral} or \textit{spatial}. Spatial networks such as \cite{velivckovic2017graph, hamilton2017inductive, atwood2015diffusion, duvenaud2015convolutional} typically extend convolution to graphs by performing a weighted average of features over neighborhoods $\mathcal{N} (u) = \{v: (u,v)\in{E}\}$. These neighborhoods are well-defined even when ${E}$ is not symmetric, so spatial methods typically have natural extensions to directed graphs. However, such simplistic extensions may miss important information in the directed graph. For example, 
filters defined using $\mathcal{N}(u)$ are not capable of assimilating the equally important information contained in $\{ v : (v, u) \in E \}$. 
Alternatively, these methods may also use the symmetrized adjacency matrix, but they cannot learn to balance  directed and undirected approaches.

In this section, we show how to extend spectral methods to  directed graphs using 
the magnetic Laplacian introduced in Section \ref{sec: mag}. To highlight the flexibility of our approach, we show how three spectral graph neural network architectures can be adapted to incorporate the magnetic Laplacian. Our approach is very general, and so for most of this section, we will perform our analysis for a general complex Hermitian, positive semidefinite matrix. However, we view the magnetic Laplacian as our primary object of interest (and use it in all of our experiments in Section \ref{sec: experiments}) because of the large body of literature studying its spectral properties and applying it to data science (see Section \ref{sec: related}).

\subsection{Spectral convolution via the magnetic Laplacian}

In this section, we let $\bm{\mathcal{L}}$ denote a Hermitian, positive semidefinite matrix, such as the normalized or unnormalized magnetic Laplacian introduced in Section \ref{sec: mag}, on a directed graph $G=(V,E)$, $|V|=N$. 
We let $\mathbf{u}_1\ldots,\mathbf{u}_N$ be an orthonormal basis of eigenvectors for $\bm{\mathcal{L}}$ and let $\mathbf{U}$ be the $N\times N$ matrix whose $k$-th column is $\mathbf{u}_k$. 
We define the directed graph Fourier transform for a signal $\mathbf{x}: V \rightarrow \mathbb{C}$ by $\widehat{\mathbf{x}}=\mathbf{U}^\dagger \mathbf{x}$, so that  
   $ \widehat{\mathbf{x}}(k)= \langle \mathbf{x},\mathbf{u}_k\rangle \, .$
 We regard the eigenvectors $\mathbf{u}_1,\ldots,\mathbf{u}_N$ as the generalizations of discrete Fourier modes to directed graphs. Since $\bU$ is unitary, we have the Fourier inversion formula
\begin{equation}\label{eqn: inversion}
    \mathbf{x}=\mathbf{U}\widehat{\mathbf{x}} = \sum_{k=1}^N\widehat{\mathbf{x}}(k)\mathbf{u}_k \, .
\end{equation}

 In Euclidean space, convolution corresponds to pointwise multiplication in the Fourier basis. Thus, we define the 
convolution of $\mathbf{x}$ with a filter $\mathbf{y}$ in the Fourier domain by $\widehat{\mathbf{y}\ast\mathbf{x}}(k)=\widehat{\mathbf{y}}(k)\widehat{\mathbf{x}}(k)$. 
By \eqref{eqn: inversion}, this implies
    $\mathbf{y}\ast\mathbf{x} = \mathbf{U}\text{Diag}(\widehat{\mathbf{y}})\widehat{\mathbf{x}} = (\mathbf{U}\text{Diag}(\widehat{\mathbf{y}})\mathbf{U}^\dagger)\mathbf{x}$,  
and so we say $\mathbf{Y}$ is a convolution matrix if 
\begin{equation}\label{eqn: conv}
    \mathbf{Y} = \mathbf{U} \bm{\Sigma} \mathbf{U}^\dagger \, ,
\end{equation}
for a diagonal matrix $\bm{\Sigma}.$ This is the natural generalization of the class of  convolutions  used in \cite{bruna:spectralNN2014}.

Next, following  \cite{Defferrard2018}  (see also \cite{hammond2011wavelets}), we show that a spectral network can be implemented in the spatial domain via polynomials of  $\bm{\mathcal{L}}$ by having   $\bm{\Sigma}$  be a polynomial of $\bm{\Lambda}$ in \eqref{eqn: conv}. This reduces the number of trainable parameters to prevent overfitting, avoids explicit diagonalization of the matrix $\bm{\mathcal{L}}$, (which is expensive for large graphs), and  improves stability 
to perturbations \cite{levie2019transferability}.
As in \cite{Defferrard2018}, we define a normalized eigenvalue matrix, with entries in $[-1,1]$, by $\widetilde{\bm{\Lambda}}=\frac{2}{\lambda_{\max}} \bm{\Lambda}-\mathbf{I}$ and assume  
\begin{equation*}
    \bm{\Sigma} = \sum_{k=0}^K\theta_k T_k(\widetilde{\bm{\Lambda}}) \, ,
\end{equation*}
for some real-valued $\theta_1,\ldots,\theta_k,$ where for $k\geq 0$, $T_k$ is the Chebyshev polynomial defined by $T_0(x)=1, T_1(x)=x,$ and $T_{k}(x)=2xT_{k-1}(x)+T_{k-2}(x)$ for $k\geq 2.$ One can use the fact that $(\mathbf{U} \widetilde{\bm{\Lambda}} \mathbf{U}^\dagger)^k = \mathbf{U}\widetilde{\bm{\Lambda}}^k\mathbf{U}^\dagger$ to see 
\begin{equation}\label{eqn: ChebNet}
    \mathbf{Y}\mathbf{x}=\mathbf{U}\sum_{k=0}^K\theta_k T_k(\widetilde{\bm{\Lambda}})\mathbf{U}^\dagger\mathbf{x} = \sum_{k=0}^K\theta_k T_k(\widetilde{\bm{\mathcal{L}}}) \mathbf{x} \, ,
\end{equation}
where, analogous to  $\widetilde{\bm{\Lambda}}$, we define  $\widetilde{\bm{\mathcal{L}}} \coloneqq \frac{2}{\lambda_{\text{max}}}\bm{\mathcal{L}}-\mathbf{I}$.  
It is important to note that, due to the complex Hermitian structure of $\widetilde{\bm{\mathcal{L}}}$, the value $\mathbf{Y} \mathbf{x} (u)$ aggregates information both from the values of $\mathbf{x}$ on $\mathcal{N}_k(u)$, the $k$-hop neighborhood of $u$, and the values of $\mathbf{x}$ on $\{ v : \mathrm{dist}(v,u) \leq k \}$, which consists of those of vertices that can reach $u$ in $k$-hops. While in an undirected graph these two sets of vertices are the same, that is not the case for general directed graphs. Furthermore, due to the difference in phase between an edge $(u,v)$ and an edge $(v,u)$, the filter matrix $\mathbf{Y}$ is also capable of aggregating information from these two sets in different ways. This capability is in contrast to any single, symmetric, real-valued matrix, as well as any matrix that encodes just $\mathcal{N}(u)$.

To obtain a network similar to \cite{kipf2016semi}, we  set $K=1$, assume that $\bm{\mathcal{L}}=\bL^{(q)}_N$, using $\lambda_{\max} \leq 2$ (see Theorem \ref{thm: normal02} in Appendix \ref{sec: eigs}) make the approximation $\lambda_{\max} \approx 2$, and set  $\theta_1=-\theta_0.$ With this, we obtain
\begin{equation*}
    \mathbf{Y}\mathbf{x} = \theta_0(\mathbf{I}+(\mathbf{D}_s^{-1/2}\mathbf{A}_s\mathbf{D}_s^{-1/2})\odot \exp (i \bTheta^{(q)}))\mathbf{x} \, .
\end{equation*} 
As in \cite{kipf2016semi}, we  substitute $(\mathbf{I}+(\mathbf{D}_s^{-1/2}\mathbf{A}_s\mathbf{D}_s^{-1/2})\odot \exp (i \bTheta^{(q)})\rightarrow \widetilde{\mathbf{D}}_s^{-1/2}\widetilde{\mathbf{A}}_s\widetilde{\mathbf{D}}_s^{-1/2}\exp (i \bTheta^{(q)}).$
This renormalization helps avoid instabilities arising from vanishing/exploding gradients and yields
\begin{equation}\label{eqn: GCN}
    \mathbf{Y}\mathbf{x} = \theta_0 \widetilde{\mathbf{D}}_s^{-1/2}\widetilde{\mathbf{A}}_s\widetilde{\mathbf{D}}_s^{-1/2}\odot \exp (i \bTheta^{(q)}) \, ,
\end{equation}
where $\widetilde{\mathbf{A}}_s=\mathbf{A}_s+\mathbf{I}$ and $\widetilde{\mathbf{D}}_s(i,i)=\sum_{j}\widetilde{\mathbf{A}}_s(i,j)$. 

\subsection{The MagNet architecture}

We now 
define our network. Let $L$ be the number of convolution layers in our network, and let $\mathbf{X}^{(0)}$ be an $N\times F_0$ input feature matrix with columns $\mathbf{x}^{(0)}_1,\ldots\mathbf{x}^{(0)}_{F_0}$.
Since our filters are   complex, we use a complex version of  ReLU  
defined by $\sigma(z)=z$, if $-\pi/2\leq \arg(z)<\pi/2$, and $\sigma(z)=0$ otherwise
(where $\arg(z)$ is the complex argument of $z \in \mathbb{C}$). 
We let $F_{\ell}$ be the number of channels in layer $\ell$, and for $1\leq\ell\leq L$, $1 \leq i \leq F_{\ell-1}$, and $1\leq j\leq F_{\ell},$  we let $\mathbf{Y}^{(\ell)}_{ij}$ be a convolution matrix defined in the sense of either \eqref{eqn: conv}, \eqref{eqn: ChebNet}, or \eqref{eqn: GCN}. To obtain the layer $\ell$ channels from the layer $\ell-1$ channels, we define  the matrix $\mathbf{X}^{(\ell)}$ with columns 
$\mathbf{x}^{(\ell)}_1,\ldots\mathbf{x}^{(\ell)}_{F_\ell}$ as:
\begin{equation}\label{eqn: single hidden layer}
    \mathbf{x}^{(\ell)}_j = \sigma \left(\sum_{i=1}^{F_{\ell-1}} \mathbf{Y}_{ij}^{(\ell)} \mathbf{x}^{(\ell-1)}_i\right).
\end{equation}
In matrix form we write 
    $\mathbf{X}^{(\ell)} = \mathbf{Z}^{(\ell)} \left(\mathbf{X}^{(\ell-1)}\right) $,
where $\mathbf{Z}^{(\ell)}$ is a hidden layer of the form \eqref{eqn: single hidden layer}. 

After the  convolutional layers, we unwind the complex $N\times F_L$ matrix $\mathbf{X}^{(L)}$ into a real-valued $N\times 2F_L$ matrix, apply a linear layer, consisting of right-multiplication by a $2F_L \times n_c$ weight matrix  $\mathbf{W}^{(L+1)}$ (where $n_c$ is the number of classes) and apply softmax. In our experiments, we set $L=2$ or $3$.  When $L=2$, our network applied to node classification, as illustrated in Figure \ref{fig: magnet}, is given by 
\begin{equation*}
    \text{softmax}(\text{unwind}(\mathbf{Z}^{(2)}(\mathbf{Z}^{(1)}(\mathbf{X}^{(0)}))) \mathbf{W}^{\mathrm{(3)}}) \, .
\end{equation*}
For link-prediction, we apply the same method through the unwind layer, and then concatenate the rows corresponding to pairs of nodes to obtain the edge features. 

\begin{figure}
    \centering
    \includegraphics[width=0.86\textwidth]{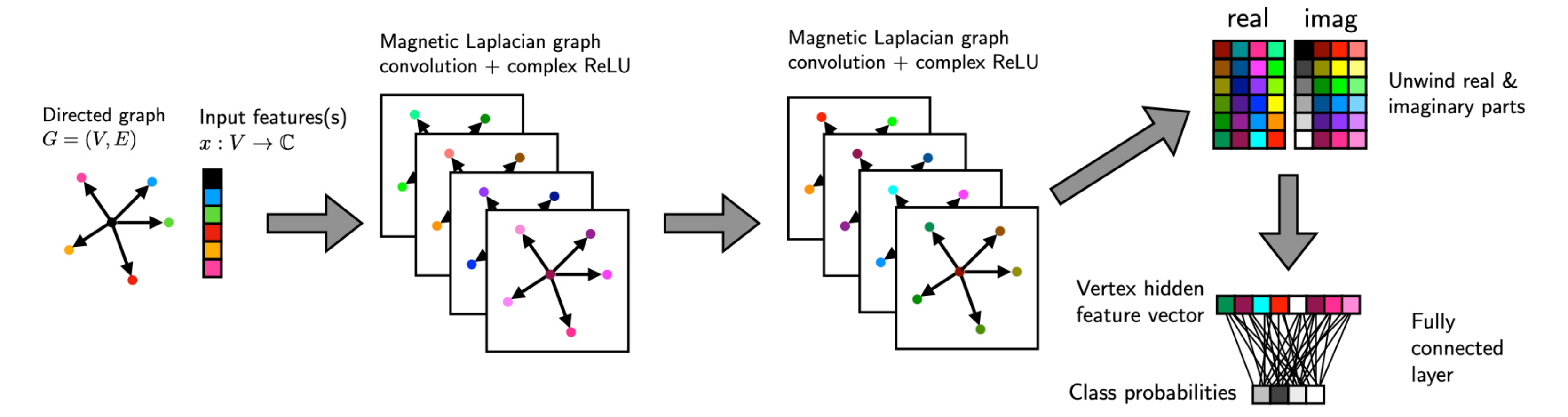}
    \caption{MagNet ($L=2$) applied to node classification. After two complex convolutional layers, we unwind the real and imaginary parts of our feature matrix and apply a fully connected layer.
    }
    \label{fig: magnet}
\end{figure}

\section{Related work}\label{sec: related}

In Section \ref{sec: other directed graph networks}, we describe other graph neural networks designed specifically for directed graphs. Notably, none of these methods encode directionality with complex numbers, instead opting for real-valued, symmetric matrices. 
In Section \ref{sec: related MagNet}, we review other work studying the magnetic Laplacian which has been studied for several decades and lately has garnered interest in the network science and graph signal processing communities. However, to the best of our knowledge, this is the first work to use it to construct a graph neural network. 
We also note there are numerous approaches to graph signal processing on directed graphs. Many of these rely on a natural analog of Fourier modes. These Fourier modes are typically defined through either a factorization of a graph shift operator or by solving an optimization problem. For further review, we refer the reader to \cite{Marques2020}.

\subsection{Neural networks for directed graphs}
\label{sec: other directed graph networks}
In \cite{ma:spectralDGCN2019}, the authors construct a 
directed Laplacian,
via identities involving the random walk matrix 
and its stationary distribution $\bPi$.
When $G$ is  undirected, one can use the fact that $\bPi$ is proportional to the degree vector to verify this directed Laplacian reduces to the standard normalized graph Laplacian. However, this method requires $G$ to be strongly connected, unlike MagNet. 
The authors of \cite{tong:directedGCN2020} use a first-order proximity matrix $\bA_F$ (equivalent to $\bA_s$ here), as well as two  second-order proximity matrices $\bA_{S_{\text{in}}}$ and $\bA_{S_{\text{out}}}$. $\bA_{S_\text{in}}$ is defined by $\bA_{S_\text{in}}(u,v)\neq 0$ if there exists a $w$ such that $(w,u),(w,v)\in E$, and  
$\bA_{S_{\text{out}}}$ is defined analogously. These three matrices collectively describe and distinguish the neighborhood of each vertex and those vertices that can reach a vertex in a single hop. 
The authors construct three different Laplacians and use a fusion operator to share information across channels. Similarly, inspired by \cite{benson2016higher}, in \cite{monti:MotifNet2018}, the authors consider several different symmetric Laplacian matrices corresponding to a number of different graph motifs.

The method of \cite{Tong2020DigraphIC} builds upon the ideas of both \cite{ma:spectralDGCN2019} and \cite{tong:directedGCN2020} and considers a directed Laplacian similar to the one used in \cite{ma:spectralDGCN2019}, but with a PageRank matrix in place of the random-walk matrix. This  allows for applications to graphs which are not strongly connected. Similar to \cite{tong:directedGCN2020}, they use higher-order receptive fields (analogous to the second-order adjacency matrices discussed above) and an inception module to share information between receptive fields of different orders. We also note \cite{klicpera2018predict}, which uses an approach based on PageRank in the spatial domain.

\subsection{Related work on the magnetic Laplacian and Hermitian adjacency matrices}
\label{sec: related MagNet}

The magnetic Laplacian has been studied since at least \cite{lieb1993fluxes}. The name originates from its interpretation as a quantum mechanical Hamiltonian of a particle under magnetic flux. Early works focused on $d$-regular graphs, where the eigenvectors of the magnetic Laplacian are equivalent to those of the Hermitian adjacency matrix. The authors of \cite{krystal:hermitianAdjDigraphs2017}, for example, show that using a complex-valued Hermitian adjacency matrix rather than the symmetrized adjacency matrix reduces the number of small, non-isomorphic cospectral graphs. 
Topics of current research into Hermitian adjacency matrices include clustering tasks  \cite{cucuringu2020hermitian} and the role of the parameter $q$ \cite{mohar2020new}.
	
The magnetic Laplacian is also the subject of ongoing research in graph signal processing  \cite{furutani:GSPdirectedHermit2019}, community detection  \cite{fanuel2017Magnetic}, and clustering \cite{CLONINGER2017370, fanuel:magneticEigenmaps2018, f2020characterization}. For example, \cite{fanuel:magneticEigenmaps2018} uses the phase of the eigenvectors to construct eigenmap embeddings analogous to \cite{belkin2003laplacian}. 
The role of $q$ is highlighted in the works of \cite{fanuel:magneticEigenmaps2018, fanuel2017Magnetic, krystal:hermitianAdjDigraphs2017, mohar2020new}, which show how 
particular choices of $q$ may highlight various graph motifs. In our context, this indicates that 
$q$ 
should be carefully tuned via cross-validation. Lastly, we note that numerous other directed graph Laplacians have been studied and applied to data science \cite{chung2005Laplacians,chung2013local, Palmer2021}. However, as alluded to in Section \ref{sec: mag}, these methods typically do not use complex Hermitian matrices.

\section{Numerical experiments}
\label{sec: experiments}

We test the performance of MagNet for node classification and link prediction on a variety of benchmark datasets as well as a directed stochastic block model. 

\subsection{Datasets}

\subsubsection{Directed Stochastic Block Model}
\label{sec: dsbm}

\begin{figure}[t]
    \centerline{
    \begin{subfigure}{0.245\textwidth}  
        \centering
        \includegraphics[width=1.11\linewidth,trim=0cm 0cm 0cm 1.5cm,clip]{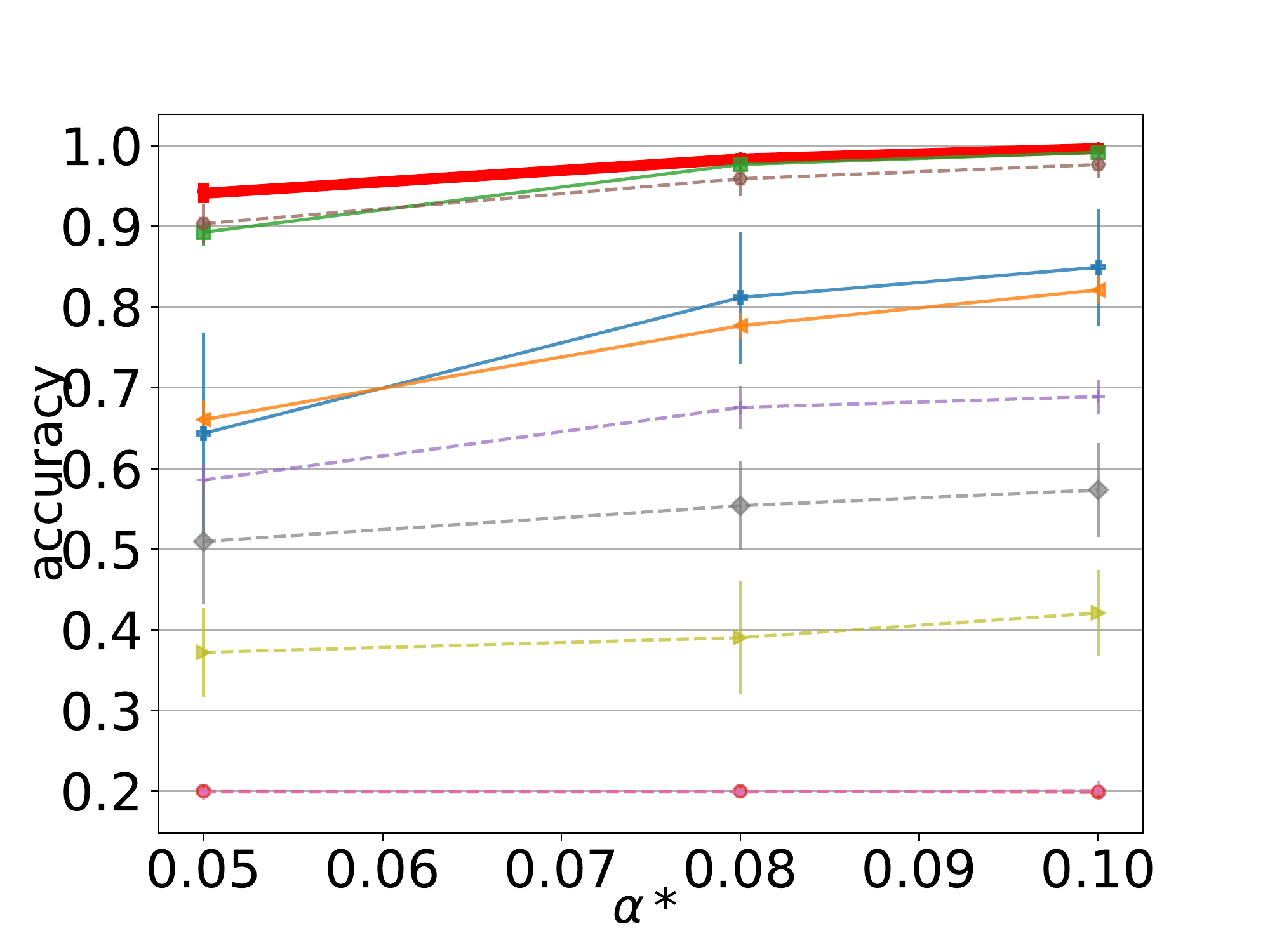} \captionsetup{width=1.0\linewidth}
        \caption{Ordered DSBM with varying edge density.} \label{fig: varying alpha}
    \end{subfigure}
    \begin{subfigure}{0.245\textwidth}
        \centering
        \includegraphics[width=1.11\linewidth,trim=0cm 0cm 0cm 1.5cm,clip]{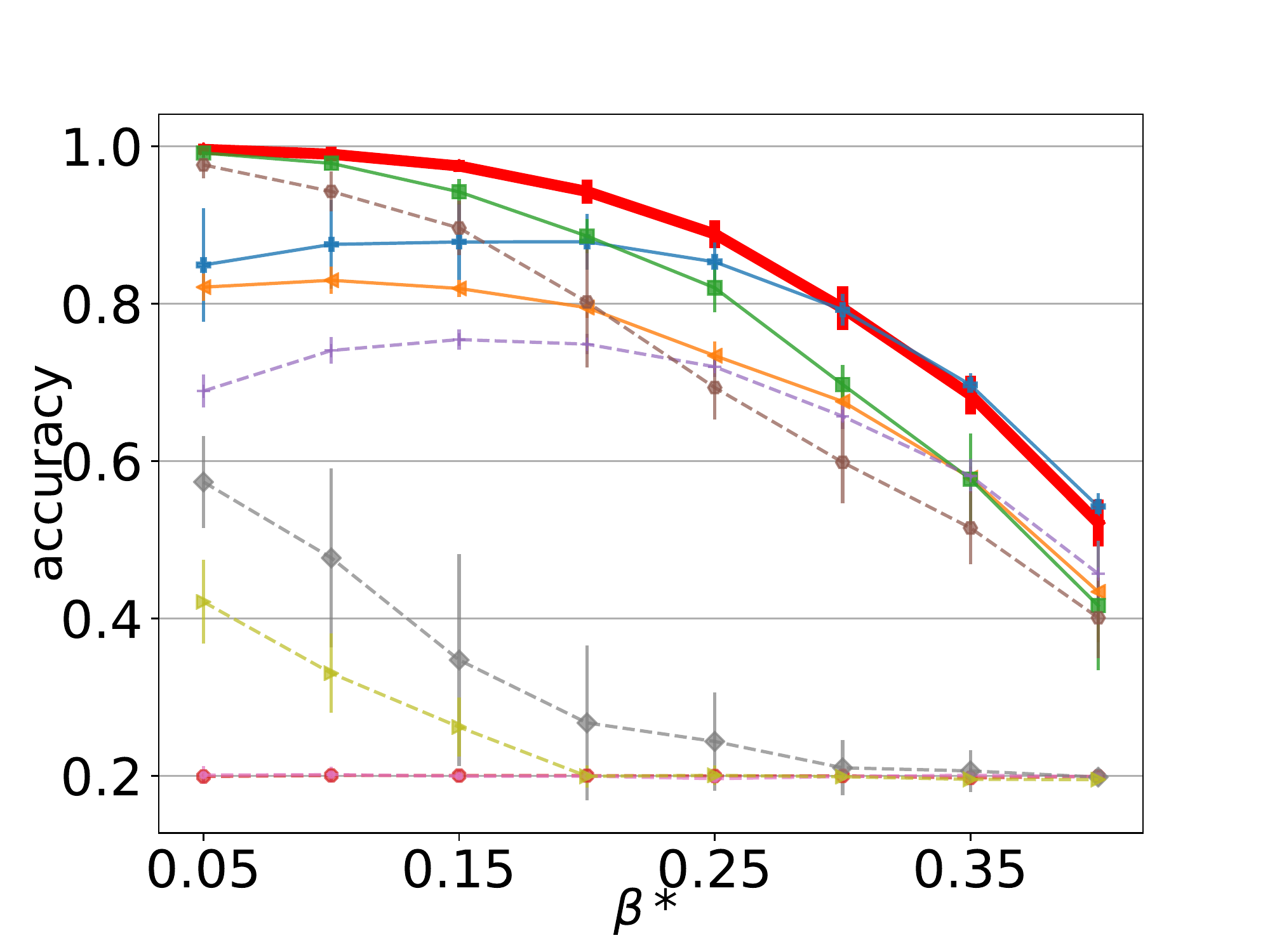} \captionsetup{width=1.0\linewidth}
        \caption{Ordered DSBM with varying net flow.}\label{fig: ordereddbsm}
    \end{subfigure}
    \begin{subfigure}{0.245\textwidth}
        \centering
        \includegraphics[width=1.11\linewidth,trim=0cm 0cm 0cm 1.5cm,clip]{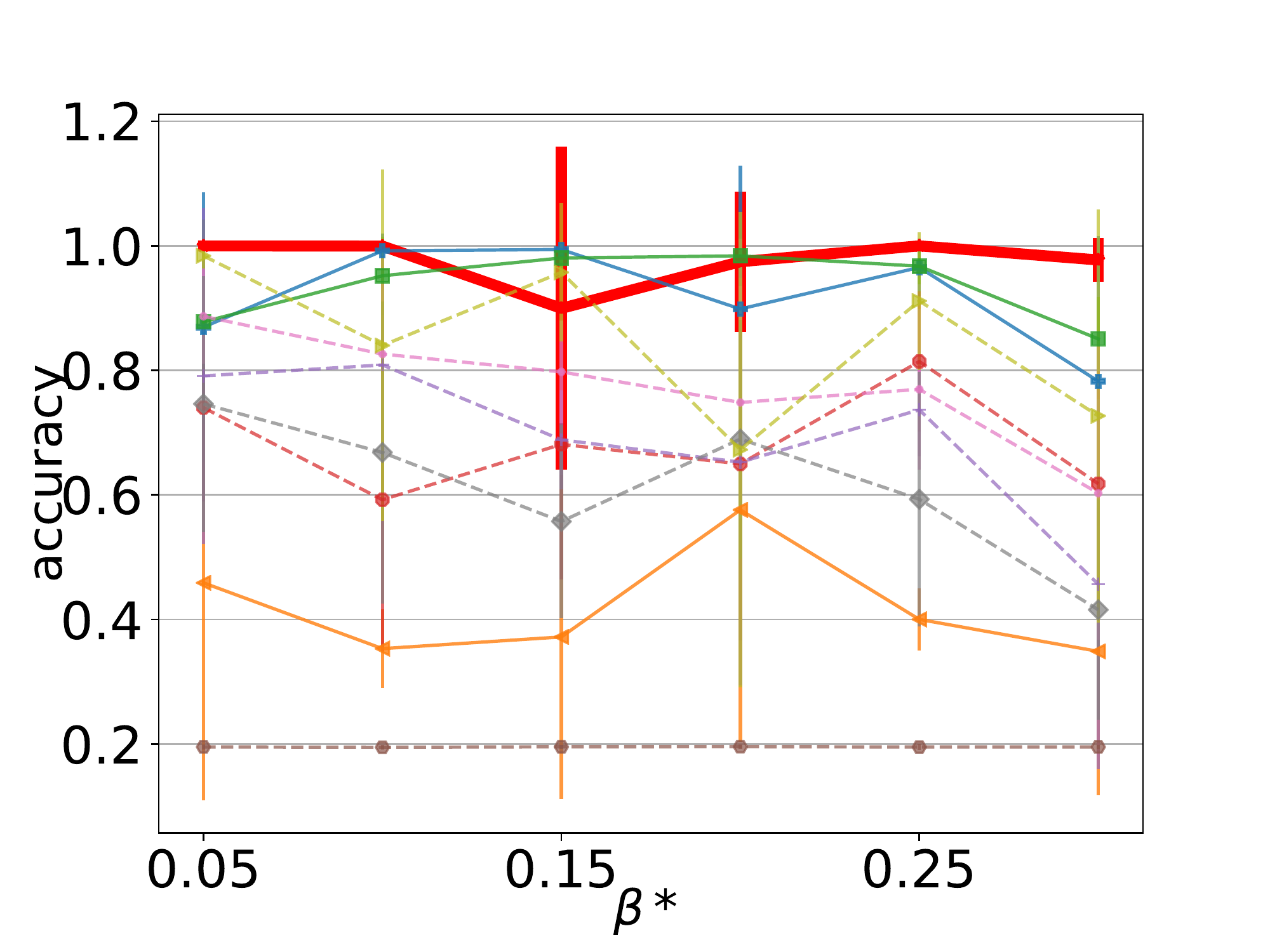} \captionsetup{width=1.0\linewidth}
        \caption{Cyclic DSBM with varying net flow.}\label{fig: cyclicdbsm}
    \end{subfigure}
    \begin{subfigure}{0.245\textwidth}
        \centering
        \includegraphics[width=1.11\linewidth,trim=0cm 0cm 0cm 1.5cm,clip]{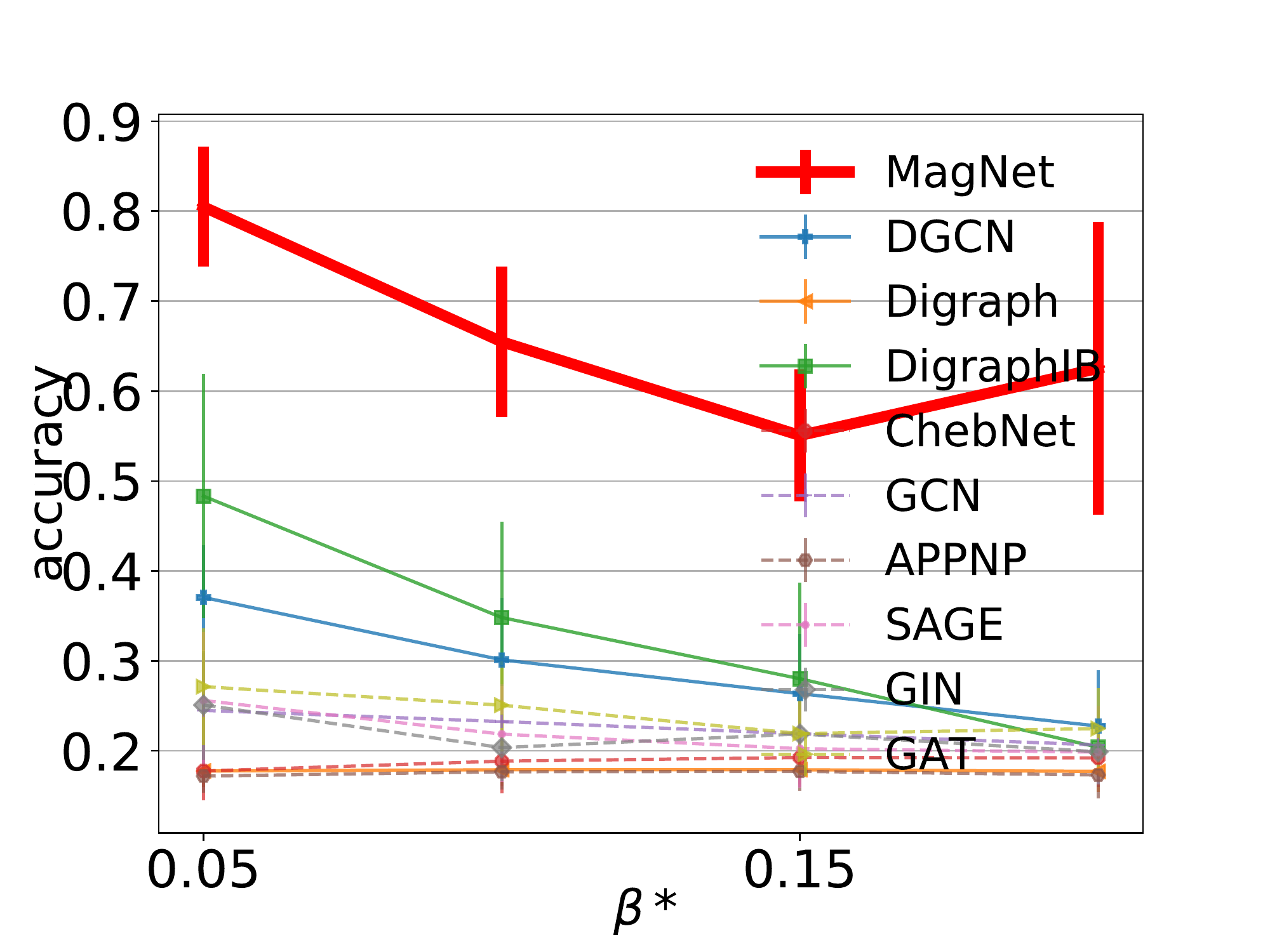} \captionsetup{width=1.0\linewidth}
        \caption{Noisy Cyclic DSBM with varying net flow.}
        \label{fig: noisycyclic}
    \end{subfigure}
    }
    \caption{Node classification accuracy. Error bars are one standard error. MagNet is bold red.}
    \label{fig:synthetic_plots}
\end{figure}

\begin{wrapfigure}{R}{.28\textwidth}
    \begin{minipage}
    {\linewidth}
    \centering\captionsetup[subfigure]{justification=centering}
    \includegraphics[width=\linewidth]{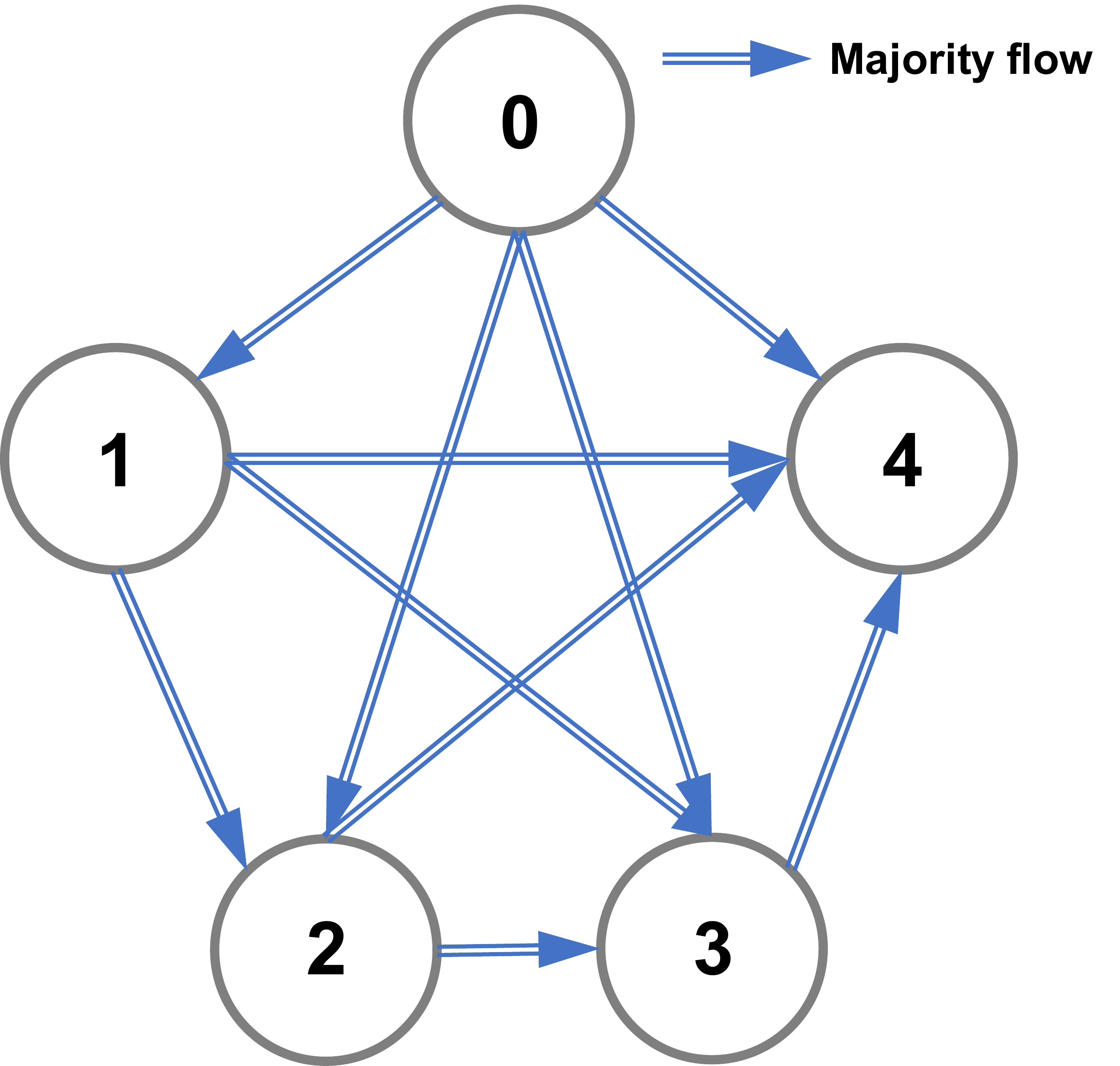}
    \subcaption{Ordered meta-graph.}
    \label{fig:completegraph}
    \includegraphics[width=\linewidth]{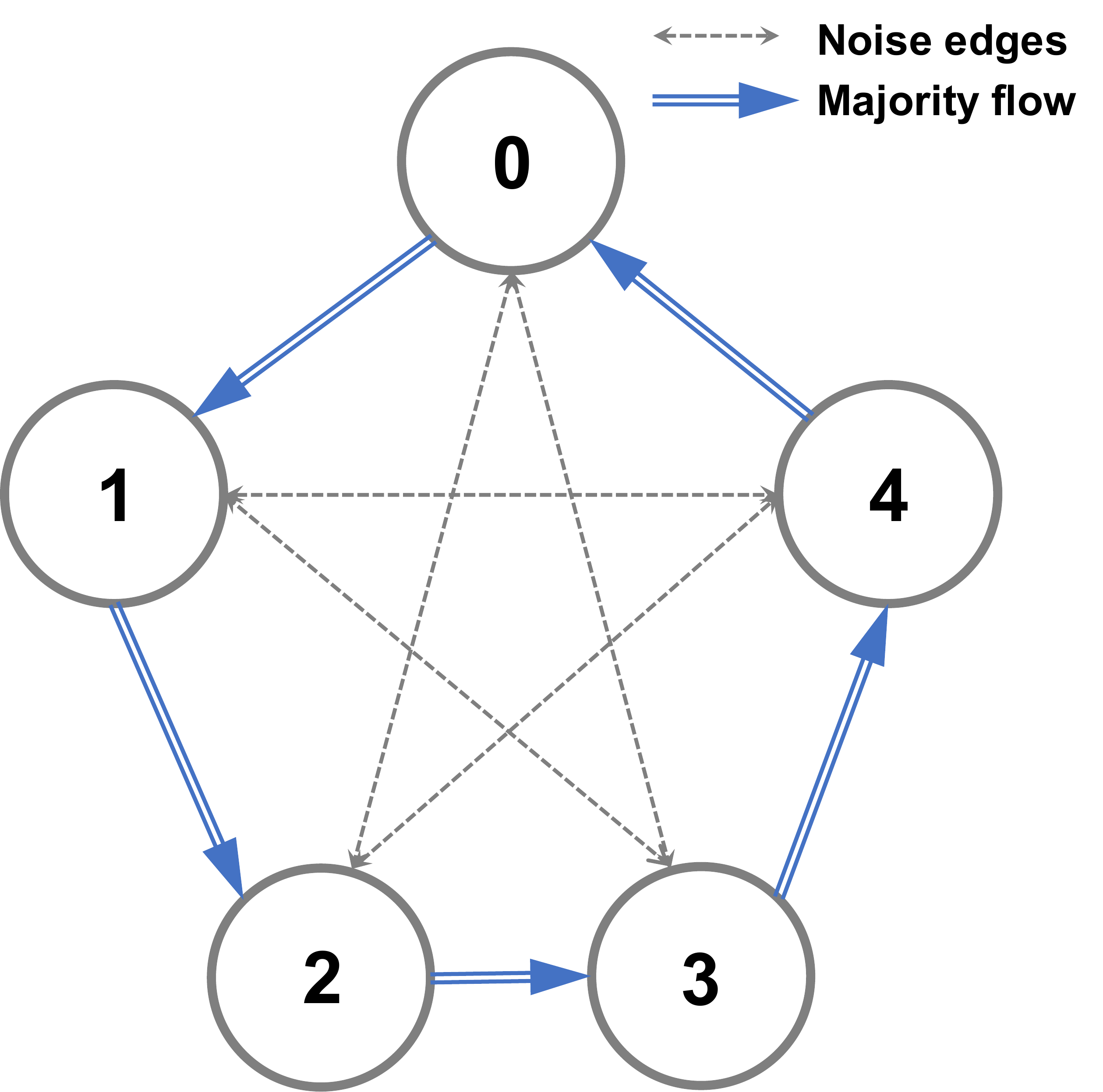}
    \subcaption{Cyclic meta-graph.}
    \label{fig:cyclicgraph}
    \end{minipage}
    \caption{Meta-graphs for the synthetic data sets.}
    \label{fig:synvis}
    \vskip -15pt
\end{wrapfigure}

We construct a directed stochastic block (DSBM) model as follows. First we divide $N$ vertices into $n_c$ equally-sized clusters $C_1,\ldots,C_{n_c}$. We define $\{\alpha_{i,j}\}_{1\leq i,j\leq n_c}$ to be a collection of probabilities,  $0<\alpha_{i,j}\leq 1$ with $\alpha_{i,j}=\alpha_{j,i}$, and for an unordered pair $u\neq v$  create an undirected edge between $u$ and $v$ with probability $\alpha_{i,j}$ if $u\in C_i,v\in C_j$. To turn this undirected graph into a directed graph, we next define $\{\beta_{i,j}\}_{1\leq i,j\leq n_c}$ to be a collection of probabilities such that $0\leq\beta_{i,j}\leq 1$ and $\beta_{i,j}+\beta_{j,i}=1$. For each undirected edge $\{u,v\}$, we assign that edge a direction by the rule that the edge points from $u$ to $v$ with probability $\beta_{i,j}$ if $u\in C_i$ and $v\in C_j$, and  points from $v$ to $u$ otherwise. We note that if $\alpha_{i,j}$ is constant, then the only way to determine the clusters will be from the directional information.  

In Figure \ref{fig:synthetic_plots}, we plot the performance of MagNet and other methods on variations of the DSBM. In each of these, we set $n_c=5$ and the goal is to classify the vertices by cluster. We set  $N=2500$, except in Figure \ref{fig: noisycyclic} where $N=500$. In Figure \ref{fig: varying alpha}, we plot the performance of our model on the DSBM with 
$\alpha_{i,j}\coloneqq \alpha^*=.1, .08$, and $.05$  for $i\neq j$, which varies the density of inter-cluster edges,  and set $\alpha_{i,i}=.1$. Here we set $\beta_{i,i}=.5$ and $\beta_{i,j}=.05$ for $i > j$. This  corresponds to the ordered meta-graph in Figure \ref{fig:completegraph}. Figure \ref{fig: ordereddbsm} also uses the ordered meta-graph, but here we fix $\alpha_{i,j}=.1$ for all $i,j$, and set $\beta_{i,j}=\beta^*,$ for $i>j$, and allow $\beta^*$ to vary from .05 to .4, which  varies the net flow from one cluster to another. The results in Figure \ref{fig: cyclicdbsm} utilize a cyclic meta-graph structure as in Figure \ref{fig:cyclicgraph} (without the gray noise edges). Specifically, we set  $\alpha_{i,j}=.1$ if $i=j$ or $i=j\pm1\mod 5$ and $\alpha_{i,j}=0$ otherwise. We define $\beta_{i,j}=\beta^*$,  $\beta_{j,i}=1-\beta^*$ when $j=(i-1)\mod 5$, and $\beta_{i,j}=0$ otherwise. In Figure \ref{fig: noisycyclic} we add noise to the cyclic structure of our meta-graph by setting $\alpha_{i,j}=.1$ for all $i,j$ and $\beta_{i,j} = .5$ for all $(i,j)$ connected by a gray edge in Figure \ref{fig:cyclicgraph} (keeping $\beta_{i,j}$ the same as in Figure \ref{fig: cyclicdbsm} for the blue edges).

\subsubsection{Real datasets}

\textit{Texas}, \textit{Wisconsin}, and \textit{Cornell} are WebKB datasets modeling links between websites at different universities \cite{pei2020geom}. We use these datasets for both link prediction and node classification with nodes labeled as student, project, course, staff, and faculty in the latter case. 
\textit{Telegram}  \cite{bovet2020activity} is a pairwise influence network between $245$ Telegram channels with $8,912$ links. To the best of our knowledge, this dataset has not previously been studied in the graph neural network literature.
Labels are generated from the method discussed in \cite{bovet2020activity}, with a total of four classes. 
The datasets \textit{Chameleon} and \textit{Squirrel}   \cite{rozemberczki2019multiscale}  represent links between Wikipedia pages related to chameleons and squirrels. We use these datasets for link prediction. Likewise, \textit{WikiCS} \cite{mernyei2020wiki} is a collection of Computer Science articles, which we also use for link prediction (see the tables in Appendix \ref{sec:tables}).
\textit{Cora-ML} and \textit{CiteSeer} are popular citation networks with node labels corresponding to scientific subareas. We use the versions of these datasets provided in \cite{bojchevski2017deep}. Further details are given in the Appendix \ref{sec: datasets}.

\subsection{Training and implementation details}

Node classification is performed in a semi-supervised setting (i.e., access to the test data, but not the test labels, during training). For the datasets \textit{Cornell}, \textit{Texas}, \textit{Wisconsin}, and \textit{Telegram} we use a 60\%/20\%/20\% training/validation/test split, which might be viewed as more akin to supervised learning, because of the small graph size. For \textit{Cora-ML} and \textit{CiteSeer}, we use the same split as \cite{Tong2020DigraphIC}. For all of these datasets we use 10 random data splits. For the DSBM datasets, we generated 5 graphs randomly for each type and for each set of parameters, each with 10 different random node splits. We use $20\%$ of the nodes for validation and we vary the proportion of training samples based on the classification difficulty, using 2\%, 10\%, and 60\% of nodes per class for the ordered, cyclic, and noisy cyclic DSBM graphs, respectively, during training, and the rest for testing. Hyperpameters were selected using one of the five generated graphs, and then applied to the other four generated graphs.

In the main text, there are two types of link prediction tasks conducted for performance evaluation. The first type is to predict the edge direction of pairs of vertices $u,v$ for which either $(u,v) \in E$ or $(v,u) \in E$. The second type is existence prediction. The model is asked to predict if $(u,v) \in E$ by considering ordered pairs of vertices $(u,v)$. For both types of link prediction, we removed 15\% of edges for testing, 5\% for validation, and use the rest of the edges for training. The connectivity was maintained during splitting. 10 splits were generated randomly for each graph and the input features are in-degree and out-degree of nodes. In the appendix, we report on two additional link prediction tasks based on a three-class classification setup: $(u,v) \in E$, $(v,u) \in E$, or $(u,v), (v,u) \notin E$. Full details are provided in Appendix \ref{sec: datasets}.

In all experiments, we used the normalized magnetic Laplacian and  implement MagNet with convolution defined as in \eqref{eqn: ChebNet}, meaning that our network may be viewed as the magnetic Laplacian generalization of ChebNet. We compare with multiple baselines in three categories: (i) spectral methods: ChebNet \cite{Defferrard2018}, GCN \cite{kipf2016semi}; (ii) spatial methods: APPNP \cite{klicpera2018predict}, SAGE \cite{hamilton2017inductive}, GIN \cite{xu2018powerful},  GAT \cite{velivckovic2017graph}; and (iii) methods designed for directed graphs: DGCN \cite{tong:directedGCN2020}, and two variants of  \cite{Tong2020DigraphIC}, a basic version (DiGraph) and a version with higher order inception blocks (DiGraphIB). All baselines in the experiments have two graph convolutional layers, except for the node classification on the DSBM using the cyclic meta-graphs (Figures \ref{fig: cyclicdbsm}, \ref{fig: noisycyclic}, and \ref{fig:cyclicgraph}) for which we also tested three layers during the grid search. 
For ChebNet, we use the symmetrized adjacency matrix. 
For the spatial networks we apply both the symmetrized and asymmmetric adjacency matrix for node classification. The results reported are the better of the two results. Full details are provided in the Appendix \ref{sec:trainingdetails}.

\begin{table*}
\centering
\caption{Node classification accuracy (\%). The best results are in \textbf{bold} and the second are \uline{underlined}.}
\setlength\tabcolsep{3.8 pt}
\begin{tabular}{c c H H H c c c c c c c} 
        \toprule
        Type & Method & Syn1 & Syn2 & Syn3 & Cornell & Texas & Wisconsin & Cora-ML & CiteSeer & Telegram  \\ 
 \midrule
\multirow{2}{*}{Spectral}&ChebNet&19.8$\pm$0.3&20.1$\pm$0.6&20.2$\pm$0.8&79.8$\pm$5.0&79.2$\pm$7.5&81.6$\pm$6.3&80.0$\pm$1.8&66.7$\pm$1.6&70.2 $\pm$6.8\\
&GCN  &-&-&-&59.0$\pm$6.4&58.7$\pm$3.8&55.9$\pm$5.4&82.0$\pm$1.1&66.0$\pm$1.5&73.4 $\pm$5.8\\ \midrule
\multirow{4}{*}{Spatial}&APPNP&-&-&-&58.7$\pm$4.0&57.0$\pm$4.8&51.8$\pm$7.4&\textbf{82.6$\pm$1.4}&66.9$\pm$1.8&67.3 $\pm$3.0\\
&SAGE&-&-&-&\uline{80.0$\pm$6.1}&\textbf{84.3$\pm$5.5}&\uline{83.1$\pm$4.8}&\uline{82.3$\pm$1.2}&66.0$\pm$1.5&56.6 $\pm$6.0\\
&GIN&-&-&-&57.9$\pm$5.7&65.2$\pm$6.5&58.2$\pm$5.1&78.1$\pm$2.0&63.3$\pm$2.5&74.4 $\pm$8.1\\
&GAT&-&-&-&57.6$\pm$4.9&61.1$\pm$5.0&54.1$\pm$4.2&81.9$\pm$1.0&\uline{67.3$\pm$1.3}&72.6 $\pm$7.5\\
\midrule
\multirow{3}{*}{Directed}&DGCN&\uline{99.1$\pm$0.2}&\uline{97.3$\pm$0.3}&88.6$\pm$1.1&67.3$\pm$4.3&71.7$\pm$7.4&65.5$\pm$4.7&81.3$\pm$1.4&66.3$\pm$2.0&\textbf{90.4 $\pm$5.6}\\
&Digraph&80.7$\pm$2.2&78.1$\pm$0.8&66.5$\pm$1.1&66.8$\pm$6.2&64.9$\pm$8.1&59.6$\pm$3.8&79.4$\pm$1.8&62.6$\pm$2.2&82.0  $\pm$3.1\\
&DiGraphIB & & & & 64.4$\pm$9.0& 64.9$\pm$13.7 & 64.1$\pm$7.0 & 79.3$\pm$ 1.2 & 61.1$\pm$1.7 & 64.1$\pm$7.0\\\midrule
\multirow{2}{*}{This paper} &MagNet&\textbf{99.8$\pm$0.1}&\textbf{98.4$\pm$0.5}&\textbf{93.8$\pm$1.5}&\textbf{84.3$\pm$7.0}&\uline{83.3$\pm$6.1}&\textbf{85.7$\pm$3.2}&79.8$\pm$2.5&\textbf{67.5$\pm$1.8} &\uline{87.6 $\pm$2.9}\\ 
\cmidrule{3-11}
        & Best $q$ & x & x & x & 0.25 & 0.15 & 0.05 & 0.0 & 0.0 & 0.15\\
\bottomrule
\end{tabular}
\label{table:class}
\end{table*}

\subsection{Results}

We see that MagNet  performs well across all tasks.  As indicated in Table \ref{table:class}, our cross-validation procedure selects $q=0$  for  node classification on the citation networks \textit{Cora-ML} and \textit{CiteSeer}. This means we achieved the best performance when regarding directional information as noise, suggesting symmetrization-based methods are appropriate in the context of node classification on citation networks. This matches our intuition. For example, in \textit{Cora-ML}, the task is to classify research papers by scientific subarea. If the topic of a given paper is ``machine learning,'' then it is likely to both cite and be cited by other machine learning papers.  For all other datasets, we find the optimal value of $q$ is nonzero, indicating that directional information is important. Our network exhibits the best performance on three out of six of these datasets and is a close second on \textit{Texas} and \textit{Telegram}. We also achieve an at least  four percentage point improvement over both ChebNet and GCN on the four data sets for which $q > 0$. These networks are similar to ours but with the classical graph Laplacian. This isolates the effects of the magnetic Laplacian and shows that it is a valuable tool for encoding directional information.  
MagNet also compares favorably to non-spectral methods on the WebKB networks (\textit{Cornell}, \textit{Texas}, \textit{Wisconsin}).
Indeed, MagNet obtains a $\sim4 $\% improvement on \textit{Cornell} and a $\sim 2.5$\% improvement on \textit{Wisconsin}, while on \textit{Texas} it has the second best accuracy, close behind SAGE.
We also see the other directed methods have relatively poor performance on the WebKB networks, perhaps since these graphs are fairly small and have very few training samples. 

On the DSBM datasets, as illustrated in Figure \ref{fig:synthetic_plots} (see also the tables in Appendix \ref{sec:tables}), we see that MagNet generally performs quite well and is the best performing network in the vast majority of cases (for full details, see  Appendix \ref{sec:tables}). The networks DGCN and DiGraphIB rely on second order proximity matrices. As demonstrated in Figure \ref{fig: cyclicdbsm}, these methods are well suited for networks with a cyclic meta-graph structure since nodes in the same cluster are likely to have common neighbors. MagNet, on the other hand, does not use second-order proximity, but can still learn the clusters by stacking multiple layers together. This improves MagNet's ability to adapt to directed graphs with different underlying topologies.  
This is illustrated in Figure \ref{fig: noisycyclic} where the network has an approximately cyclic meta-graph structure. In this setting, MagNet continues to perform well, but the performance of DGCN and DiGraphIB deteriorate significantly. Interestingly, MagNet performs well on the DSBM cyclic meta-graph (Figure \ref{fig: cyclicdbsm}) with $q \approx .1$, whereas $q \geq .2$ is preferred for the other three DSBM tests; we leave a more in-depth investigation for future work. Further details are available in Appendix \ref{sec: bestq}. 

\begin{table}[t]
    \caption{Link prediction accuracy (\%). The best results are in \textbf{bold} and the second are \uline{underlined}.}
    \setlength\tabcolsep{2.8pt}
    \centerline{
    \begin{tabular}{c c H c c c c H c c c} 
        \toprule
&\multicolumn{5}{c}{Direction prediction} & \multicolumn{5}{c}{Existence prediction}\\
\cmidrule(lr){2-6} 
\cmidrule(lr){7-11}
&Cornell &Texas &Wisconsin &Cora-ML &CiteSeer&Cornell &Texas &Wisconsin &Cora-ML &CiteSeer \\\midrule
ChebNet&71.0$\pm$5.5&66.8$\pm$6.9&67.5$\pm$4.5&72.7$\pm$1.5&68.0$\pm$1.6&80.1$\pm$2.3&81.7$\pm$2.7&82.5$\pm$1.9&80.0$\pm$0.6&77.4$\pm$0.4\\
GCN&56.2$\pm$8.7&69.8$\pm$4.9&71.0$\pm$4.0&79.8$\pm$1.1&68.9$\pm$2.8&75.1$\pm$1.4&76.1$\pm$3.0&75.1$\pm$1.9&81.6$\pm$0.5&76.9$\pm$0.5\\\midrule
APPNP&69.5$\pm$9.0&76.8$\pm$5.1&75.1$\pm$3.5&\uline{83.7$\pm$0.7}&77.9$\pm$1.6&74.9$\pm$1.5&76.4$\pm$2.5&75.7$\pm$2.2&\uline{82.5$\pm$0.6}&78.6$\pm$0.7\\
SAGE&75.2$\pm$11.0&69.8$\pm$5.9&72.0$\pm$3.5&68.2$\pm$0.8&68.7$\pm$1.5&79.8$\pm$2.4&75.2$\pm$3.1&77.3$\pm$2.9&75.0$\pm$0.0&74.1$\pm$1.0\\
GIN&69.3$\pm$6.0&76.1$\pm$4.5&74.8$\pm$3.7&83.2$\pm$0.9&76.3$\pm$1.4&74.5$\pm$2.1&77.5$\pm$3.8&76.2$\pm$1.9&\uline{82.5$\pm$0.7}&77.9$\pm$0.7\\
GAT&67.9$\pm$11.1&50.0$\pm$2.0&53.2$\pm$2.6&50.0$\pm$0.1&50.6$\pm$0.5&77.9$\pm$3.2&74.9$\pm$0.3&74.6$\pm$0.0&75.0$\pm$0.0&75.0$\pm$0.0\\\midrule
DGCN&\textbf{80.7$\pm$6.3}&72.5$\pm$8.0&74.5$\pm$7.2&79.6$\pm$1.5&78.5$\pm$2.3&80.0$\pm$3.9&82.3$\pm$3.1&\uline{82.8$\pm$2.0}&82.1$\pm$0.5&\uline{81.2$\pm$0.4}\\
DiGraph&79.3$\pm$1.9&79.8$\pm$3.0&\uline{82.3$\pm$4.9}&80.8$\pm$1.1&81.0$\pm$1.1&\textbf{80.6$\pm$2.5}&82.8$\pm$2.5&\uline{82.8$\pm$2.6}&81.8$\pm$0.5&\textbf{82.2$\pm$0.6}\\
DiGraphIB&79.8$\pm$4.8&81.1$\pm$2.5&82.0$\pm$4.9&83.4$\pm$1.1&\uline{82.5$\pm$1.3}&\uline{80.5$\pm$3.6}&83.6$\pm$2.6&82.4$\pm$2.2&82.2$\pm$0.5&81.0$\pm$0.5\\\midrule
MagNet&\textbf{80.7$\pm$2.7}&79.5$\pm$3.7&\textbf{83.6$\pm$2.8}&\textbf{86.1$\pm$0.9}&\textbf{85.1$\pm$0.8}&\textbf{80.6$\pm$3.8}&83.8$\pm$3.3&\textbf{82.9$\pm$2.6}&\textbf{82.8$\pm$0.7}&79.9$\pm$0.5\\
\cmidrule{2-11}
Best $q$ & 0.10 & 0.10 & 0.05 & 0.05 & 0.15 & 0.25 & 0.10 & 0.25 & 0.05 & 0.05\\[0.5ex]
\bottomrule
\end{tabular}
}
\label{table:link_pred}
\end{table}

For link prediction, we achieve the best performance on seven out of eight tests 
as shown in  Table \ref{table:link_pred}. 
We also note that Table \ref{table:link_pred} reports optimal non-zero $q$ values for each task. This indicates that incorporating directional information is important for link prediction, even on citation networks such as \textit{Cora} and \textit{CiteSeer}.  This matches our intuition, since there is a clear difference between a paper with many citations and one with many references. 
More results on different datasets, and closely related tasks (including a three-class classification problem), are provided in Appendix \ref{sec:tables}.

\section{Conclusion}
\label{sec: conclusion}

We have introduced MagNet, a neural network for directed graphs based on the magnetic Laplacian. This network can be viewed as the natural extension of spectral graph convolutional networks to the directed graph setting. We demonstrate the effectiveness of our network, and the importance of incorporating directional information via a complex Hermitian matrix, for link prediction and node classification on both real and synthetic datasets. Interesting avenues of future research would be using multiple $q$'s along different channels, exploring the role of different normalizations of the magnetic Laplacian, and incorporating the magnetic Laplacian into other network architectures. 


\noindent\textbf{Limitations and Ethical Considerations:} Our method has natural extensions to weighted, directed graphs when all edges are directed. However, it not clear what is the best way to extend it to weighted mixed graphs (with both directed and undirected edges). Our network does not incorporate an attention mechanism and, similar to many other networks, is not scalable to large graphs in its current form (although this may be addressed in future work). All of our data is publicly available for research purposes and does not contain personally identifiable information or offensive content. The method presented here has no greater or lesser impact on society than other graph neural network algorithms.

\section*{Acknowledgments}

We would like to thank Jie Zhang who pointed out that our definition of the magnetic Laplacian differed by an entry-wise complex conjugation from the most commonly used definition in the literature. Y.H. thanks her supervisors Mihai Cucuringu and Gesine Reinert for their guidance.

This work was supported by the National Institutes of Health [grant NIGMS-R01GM135929 to M.H. and supporting X.Z, N.B.]; the University of Oxford [the Clarendon scholarship to Y.H.]; the National Science Foundation [grants DMS-1845856 and DMS-1912906 to M.H.]; and the Department of Energy [grant DE-SC0021152 to M.H.].

\FloatBarrier

\appendix 
\section{Github repository} 
\label{gen_inst}

A Github repository containing code needed to reproduce the results is \url{https://github.com/matthew-hirn/magnet}.

\section{List of method abbreviations}
\label{sec: acronyms}

\begin{multicols}{2}
\begin{itemize}[topsep=0pt, itemsep=0pt]
    \item MagNet (this paper)
    \item ChebNet \cite{Defferrard2018}
    \item GCN \cite{kipf2016semi}
    \item APPNP \cite{klicpera2018predict}
    \item GAT \cite{velivckovic2017graph}
    \item SAGE \cite{hamilton2017inductive}
    \item GIN \cite{xu2018powerful}
    \item DGCN \cite{tong:directedGCN2020}
    \item DiGraph \cite{Tong2020DigraphIC} %
    \item DiGraphIB \cite{Tong2020DigraphIC}: DiGraph with inception blocks
\end{itemize}
\end{multicols}

\section{Further implementation details}\label{sec:trainingdetails}

We set the parameter $K=1$ in our implementation of both ChebNet and MagNet. We train all models with a maximum of $3000$ epochs and stop early if the validation error doesn't decrease after $500$ epochs for both node classification and link prediction tasks. One dropout layer with a probability of $0.5$ is created before the last linear layer. The model is picked with the best validation accuracy during training for testing.  We tune the number of filters in [$16, 32, 48$] for the graph convolutional layers for all models, except DigraphIB, since the inception block has more trainable parameters. 
For node classification, we tune the learning rate in [$1e^{-3}, 5e^{-3}, 1e^{-2}$] for all models. Compared with node classification, the number of available samples for link prediction is much larger. Thus, we set a relatively small learning rate of $1e^{-3}$.

We use Adam as the optimizer and $\ell_2$ regularization with the hyperparameter set as $5e^{-4}$ to avoid overfitting. We post the best testing performance by grid-searching based on validation accuracy. For node classification on the synthetic datasets, we generate a one-dimensional node feature sampled from the standard normal distribution. We use the original features for the other node classification datasets. For link prediction, we use the in-degree and out-degree as the node features for all datasets instead the original features. This allows all models to learn directed information from the adjacency matrix. Our experiments were conducted on 8 compute nodes each with 1 Nvidia Tesla V100 GPU, $120$G RAM, and $32$ Intel Xeon E5-2660 v3 CPUs; as well as on a compute node with 8 Nvidia RTX 8000 GPUs, $1000$GB RAM, and 48 Intel Xeon Silver 4116 CPUs.

Here are implementation details specific to certain methods:
\begin{itemize}[itemsep=0pt]
    \item We set the parameter $\epsilon$ to $0$ in GIN for both tasks.
    \item For GAT, the number of heads tuned is in [$2, 4, 8$].
    \item For APPNP, we set $K=10$ for node classification (following the original paper  \cite{klicpera2018predict}), and search K in [$1,5,10$] for link prediction.
    \item The coefficient $\alpha$ for PageRank-based models (APPNP, DiGraph) is searched in [$0.05, 0.1, 0.15, 0.2$].
    \item For DiGraph, the model includes graph convolutional layers without the high-order approximation and inception module. The high order Laplacian and the inception module is included in DigraphIB.
    \item DigraphIB is a bit different than other networks because it requires generating a three-channel Laplacian tensor. For this network, the number of filters for each channel is searched in [$6, 11, 21$] for node classification and link prediction.
    \item For GCN, the out-degree normalized, directed adjacency matrix, including self-loops is also tried in addition to the symmetrized adjacency matrix for node classification tasks, except for synthetic datasets since symmetrization will break the cluster pattern. 
    \item For other spatial methods, including APPNP, GAT, SAGE, and GIN, we tried both the symmetrized adjacency matrices and the original directed adjacency matrices for node classification tasks except for synthetic datasets.
\end{itemize}

\section{Datasets}\label{sec: datasets}

\subsection{Node classification}

As shown in Table \ref{tab:classdata}, we use six real datasets for node classification. A directed edge is defined as follows. If the edge $(u,v) \in E$ but $(v,u) \notin E$, then $(u,v)$ is a directed edge. If $(u,v) \in E$ and $(v,u)\in E$, then $(u,v)$ and $(v,u)$ are undirected edges (in other words, undirected edges that are not self-loops are counted twice). For the citation datasets, \textit{Cora-ML} and \textit{Citeseer}, we randomly select $20$ nodes in each class for training, $500$ nodes for validation, and the rest for testing following \cite{Tong2020DigraphIC}. For the synthetic datasets (\textit{ordered DSBM graphs}, \textit{cyclic DSBM graphs}, \textit{noisy cyclic DSBM graphs}), we generate a one-dimensional node feature sampled from the standard normal distribution.

Ten folds are generated randomly for each dataset, except for \textit{Cornell}, \textit{Texas} and \textit{Wisconsin}. 
For \textit{Cornell}, \textit{Texas}, and \textit{Wisconsin}, we use the same training, validation, and testing folds as \cite{pei2020geom}.
For \textit{Telegram}, we treat it as a directed, unweighted graph and randomly generate 10 splits for training/validation/testing with 60\%/20\%/20\% of the nodes. The node features are sampled from the normal distribution.

\begin{table*}[h!]
\centering
\caption{Real datasets for node classification.}
\begin{tabular}{l H H H c c c c c c } 
\toprule
& H & H & H & Cornell & Texas & Wisconsin & Cora-ML & Citeseer & Telegram \\  \midrule
$\#$ Nodes    & 2500 & 2500 & 2500 & 183 & 183 & 251 & 2,995 & 3,312 & 245\\
$\#$ Edges    & 312,635 & 261,667 & 186,984 & 295 & 309 & 499 & 8,416 & 4,715 & 8,912\\
\% Directed edges &&&&86.9 &76.6 &77.9 & 93.9 & 95.0 &82.4 \\
$\#$ Features & 1 & 1 & 1 & 1,703 & 1,703 & 1,703 & 2,879 & 3,703 & 1\\
$\#$ Classes  & 5 & 5 & 5 & 5 & 5 & 5 & 7 & 6 & 4\\
\bottomrule
\end{tabular}
\label{tab:classdata}
\end{table*}

\subsection{Link prediction}
\label{sec: link prediction}

We use eight real datasets in link prediction as demonstrated in Table \ref{tab:linkdata}. Instead of using the original features, we use the in-degree and out-degree as the node features in order to allow the models to learn structural information from the adjacency matrix directly. The connectivity is maintained by getting the undirected minimum spanning tree before removing edges for validation and testing. For the results in the main text, undirected edges and, if they exist, pairs of vertices with multiple edges between them, may be placed in the training/validation/testing sets. However, labels that indicate the direction of such edges are not well defined, and therefore can be considered as noisy labels from the machine learning perspective. In order to obtain a full set of well-defined, noiseless labels, we also run experiments in which undirected edges and pairs of vertices with multiple edges between them are ignored when sampling edges for training/validation/testing (in other words, only directed edges, and the absence of an edge, are included). We evaluated all models on four prediction tasks, which we now describe.

To construct the datasets that we use for training, validation and testing, which consist of pairs of vertices in the graph, we do the following. (1) Existence prediction. If $(u,v) \in E$, we give $(u,v)$ the label 0, otherwise its label is 1. The proportion of the two classes of edges is 25\% and 75\%, respectively, when undirected edges and multi-edges are included, and 50\% and 50\%, respectively, when only directed edges are included. (2) Direction prediction. Given an ordered node pair $(u,v)$, we give the label 0 if $(u,v) \in E$ and the label 1 if $(v, u) \in E$, conditioning on $(u,v) \in E$ or $(v,u) \in E$. The proportion of the two types of edges is 50\% and 50\%. (3) Three-class link prediction. For a pair of ordered nodes $(u,v)$, if $(u,v) \in E$, we give the label 0, if $(v,u) \in E$, we give the label 1, if $(u,v) \notin E$ and $(v,u)\notin E$, we give the label 2. The proportion of the three types of edges is 25\%, 25\%, and 50\%. (4) Direction prediction by three classes training. This task is based on the training of task (3). We only evaluate the performance with ordered node pairs $(u,v)$ when $(u,v) \in E$ or $(v,u) \in E$. We randomly generated ten folds for all datasets. We used 15\% and 5\% of edges for testing and validation for all datasets. The classification results are in Appendix \ref{sec:tables}.

\begin{table*}[ht]
\caption{Datasets for link prediction.}
\setlength\tabcolsep{2.5pt}
\centerline{
\begin{tabular}{l c c c c c c H c c } 
\toprule
              & Cornell & Texas & Wisconsin & Cora-ML    & CiteSeer & WikiCS  & Email  & Chameleon & Squirrel \\
\midrule
$\#$ Nodes  & 183& 183& 251 & 2,995    & 3,312   & 11,701  & 1,005  & 2,277 & 5,201\\
$\#$ Edges  &295&309&499  & 8,416 & 4,715  & 216,123 & 25,571 & 36,101 & 217,073\\
\% Directed edges &86.9 &76.6 &77.9 & 93.9 & 95.0 & 45.9 & &73.9&82.8 \\
$\#$ Features &2&2&2 & 2 & 2 & 2 & 2 & 2  & 2 \\
\bottomrule
\end{tabular}
}
\label{tab:linkdata}
\end{table*}

\section{Eigenvalues of the magnetic Laplacian}\label{sec: eigs}
In this section we state and prove two theorems.
Theorem \ref{thm: posdef}, which shows that both the normalized and unnormalized magnetic Laplacian a postive semidefinite, is well known (see e.g. \cite{fanuel:magneticEigenmaps2018}). Theorem \ref{thm: normal02}, which shows that the eigenvalues of the normalized magnetic Laplacian lie in the interval $[0,2]$, is a straightforward adaption of the corresponding result for the traditional normalized graph Laplacian. We give full proofs of both results for completeness.
\begin{theorem}\label{thm: posdef}
Let $G = (V, E)$ be a directed graph where $V$ is a set of $N$ vertices and $E\subseteq V\times  V$ is a set of directed edges. Then,
for all $q\geq 0,$ both the unnormalized magnetic Laplacian $\bL_{U}^{(q)}$ and its normalized counterpart 
$\bL_{N}^{(q)}$ are positive semidefinite.
\end{theorem}

\begin{proof}
Let $\mathbf{x}\in\mathbb{C}^N.$ We first note that since $\mathbf{L}^{(q)}_U$ is Hermitian we have $\text{Imag}(\mathbf{x}^\dagger\mathbf{L}^{(q)}_U\mathbf{x})=0$. Next, we use the definition of $\mathbf{D}_s$ and the fact that $\mathbf{A}_s$ is symmetric to observe that 
\begin{align}
    &2\text{Real}\left(\mathbf{x}^\dagger\mathbf{L}^{(q)}_U\mathbf{x}\right)\nonumber\\=&2\sum_{u,v=1}^N \mathbf{D}_s(u,v)\mathbf{x}(u)\overline{\mathbf{x}(v)}-2\sum_{u,v=1}^N \mathbf{A}_s(u,v)\mathbf{x}(u)\overline{\mathbf{x}(v)}\cos(i \bTheta^{(q)} (u,v))\nonumber\\
    =&2\sum_{u=1}^N \mathbf{D}_s(u,u)\mathbf{x}(u)\overline{\mathbf{x}(u)}-2\sum_{u,v=1}^N \mathbf{A}_s(u,v)\mathbf{x}(u)\overline{\mathbf{x}(v)}\cos(i \bTheta^{(q)} (u,v))\nonumber\\
    =&2\sum_{u,v=1}^N \mathbf{A}_s(u,v)|\mathbf{x}(u)|^2-2\sum_{u,v=1}^N \mathbf{A}_s(u,v)\mathbf{x}(u)\overline{\mathbf{x}(v)}\cos(i \bTheta^{(q)} (u,v))\nonumber\\
    =&\sum_{u,v=1}^N \mathbf{A}_s(u,v)|\mathbf{x}(u)|^2+
    \sum_{u,v=1}^N \mathbf{A}_s(v,u)|\mathbf{x}(v)|^2-2\sum_{u,v=1}^N \mathbf{A}_s(u,v)\mathbf{x}(u)\overline{\mathbf{x}(v)}\cos(i \bTheta^{(q)} (u,v))\nonumber\\
        =&\sum_{u,v=1}^N \mathbf{A}_s(u,v)|\mathbf{x}(u)|^2+
    \sum_{u,v=1}^N \mathbf{A}_s(u,v)|\mathbf{x}(v)|^2-2\sum_{u,v=1}^N \mathbf{A}_s(u,v)\mathbf{x}(u)\overline{\mathbf{x}(v)}\cos(i \bTheta^{(q)} (u,v))\nonumber\\
    =&\sum_{u,v=1}^N \mathbf{A}_s(u,v)\left(|\mathbf{x}(u)|^2+
    |\mathbf{x}(v)|^2-2\mathbf{x}(u)\overline{\mathbf{x}(v)}\cos(i \bTheta^{(q)} (u,v))\right)\label{eqn: quad form}\\
    \geq&\sum_{u,v=1}^N \mathbf{A}_s(u,v)\left(|\mathbf{x}(u)|^2+
    |\mathbf{x}(v)|^2-2|\mathbf{x}(u)||\mathbf{x}(v)|\right)\nonumber\\
    =&\sum_{u,v=1}^N \mathbf{A}_s(u,v)(|\mathbf{x}(u)|-|\mathbf{x}(v)|)^2\nonumber\\
    \geq&0.\nonumber
\end{align}
Thus, $\mathbf{L}_U^{(q)}$ is positive semidefinite. For the normalized magnetic Laplacian, we note that 
\begin{equation*}
 \left(\bD_s^{-1/2}\bA_s\bD_s^{-1/2}\right) \odot \exp (i \bTheta^{(q)})
 =\bD_s^{-1/2}\left(\bA_s\odot \exp (i \bTheta^{(q)})\right)\bD_s^{-1/2},
\end{equation*}
and therefore
\begin{equation}\label{eqn: divide by D}
    \mathbf{L}^{(q)}_N=\bD_s^{-1/2}\mathbf{L}^{(q)}_U\bD_s^{-1/2}.
\end{equation}
Thus, letting $\mathbf{y}=\mathbf{D}_s^{-1/2}\mathbf{x},$ the fact that $\mathbf{D}_s$ is diagonal implies
\begin{equation*}
\mathbf{x}^\dagger\mathbf{L}^{(q)}_N\mathbf{x}=\mathbf{x}^\dagger\bD_s^{-1/2}\mathbf{L}^{(q)}_U\bD_s^{-1/2}\mathbf{x}=\mathbf{y}^\dagger\mathbf{L}^{(q)}_U\mathbf{y}\geq 0.
\end{equation*}
\end{proof}

\begin{theorem}\label{thm: normal02}Let $G = (V, E)$ be a directed graph where $V$ is a set of $N$ vertices and $E\subseteq V\times  V$ is a set of directed edges. Then,    for all $q\geq 0$,
the eigenvalues of the normalized magnetic Laplacian $\mathbf{L}_N^{(q)}$ are contained in the interval $[0,2]$.
\end{theorem}

\begin{proof}
By Theorem \ref{thm: posdef}, we know that $\mathbf{L}^{(q)}_N$ has real, nonnegative eigenvalues. Therefore, we need to show that the lead eigenvalue, $\lambda_{N},$ is less than or equal to 2.
The Courant-Fischer theorem shows that 
\begin{equation*}
    \lambda_N=\max_{\mathbf{x}\neq 0}\frac{\mathbf{x}^\dagger\mathbf{L}^{(q)}_N\mathbf{x}}{\mathbf{x}^\dagger\mathbf{x}}.
\end{equation*}
Therefore, using \eqref{eqn: divide by D} and setting $\mathbf{y}=\mathbf{D}_s^{-1/2}\mathbf{x}$, we have 
\begin{equation*}
    \lambda_N=\max_{\mathbf{x}\neq 0}\frac{\mathbf{x}^\dagger\mathbf{D}_s^{-1/2}\mathbf{L}^{(q)}_U\mathbf{D}_s^{-1/2}\mathbf{x}}{\mathbf{x}^\dagger\mathbf{x}} =\max_{\mathbf{y}\neq 0}\frac{\mathbf{y}^\dagger\mathbf{L}^{(q)}_U\mathbf{y}}{\mathbf{y}^\dagger\mathbf{D}_s\mathbf{y}}.
\end{equation*}
First, we observe that since $\mathbf{D}_s$ is diagonal, we have 
\begin{equation*}
    \mathbf{y}^\dagger\mathbf{D}_s\mathbf{y}=\sum_{u,v=1}^N \mathbf{D}_s(u,v)\mathbf{y}(u)\overline{\mathbf{y}(v)}=\sum_{u=1}^N \mathbf{D}_s(u,u)|\mathbf{y}(u)|^2
\end{equation*}
Next, we note that by \eqref{eqn: quad form}, we have \begin{align*}
    \mathbf{y}^\dagger\mathbf{L}^{(q)}_U\mathbf{y} &= \frac{1}{2}\sum_{u,v=1}^N \mathbf{A}_s(u,v)\left(|\mathbf{x}(u)|^2+
    |\mathbf{x}(v)|^2-2\mathbf{x}(u)\overline{\mathbf{x}(v)}\cos(i \bTheta^{(q)} (u,v))\right)\\
    &\leq \frac{1}{2}\sum_{u,v=1}^N \mathbf{A}_s(u,v)(|\mathbf{x}(u)|+|\mathbf{x}(v)|)^2\\
    &\leq  \sum_{u,v=1}^N \mathbf{A}_s(u,v)(|\mathbf{x}(u)|^2+|\mathbf{x}(v)|^2).
\end{align*}
Therefore, since $\mathbf{A_s}$ is symmetric, we have 
\begin{align*}
    \mathbf{y}^\dagger\mathbf{L}^{(q)}_U\mathbf{y} 
    &\leq 2\sum_{u,v=1}^N \mathbf{A}_s(u,v)|\mathbf{x}(u)|^2\\
    &=2\sum_{u=1}^N|\mathbf{x}(u)|^2\left(\sum_{v=1}^N \mathbf{A}_s(u,v)\right)\\
    &=2\sum_{u=1}^N\mathbf{D}_s(u,u)|\mathbf{x}(u)|^2\\
    &=2    \mathbf{y}^\dagger\mathbf{D}_s\mathbf{y}.
\end{align*}
\end{proof}

\section{The eigenvectors and eigenvalues of directed stars and cycles}\label{sec: star}

In this section, we examine the eigenvectors and eigenvalues of the unnormalized magnetic Laplacian on two example graphs. As alluded to in the main text, in the directed star graph directional information is contained in the eigenvectors only. For the directed cycle, on the other hand, the magnetic Laplacian encodes the directed nature of the graph only through the eigenvalues. Both examples can be verified via direct pen and paper calculation. 
\begin{figure}
    \centering
    \begin{subfigure}{0.495\textwidth}  
    	\centering
    	\includegraphics[scale=0.15]{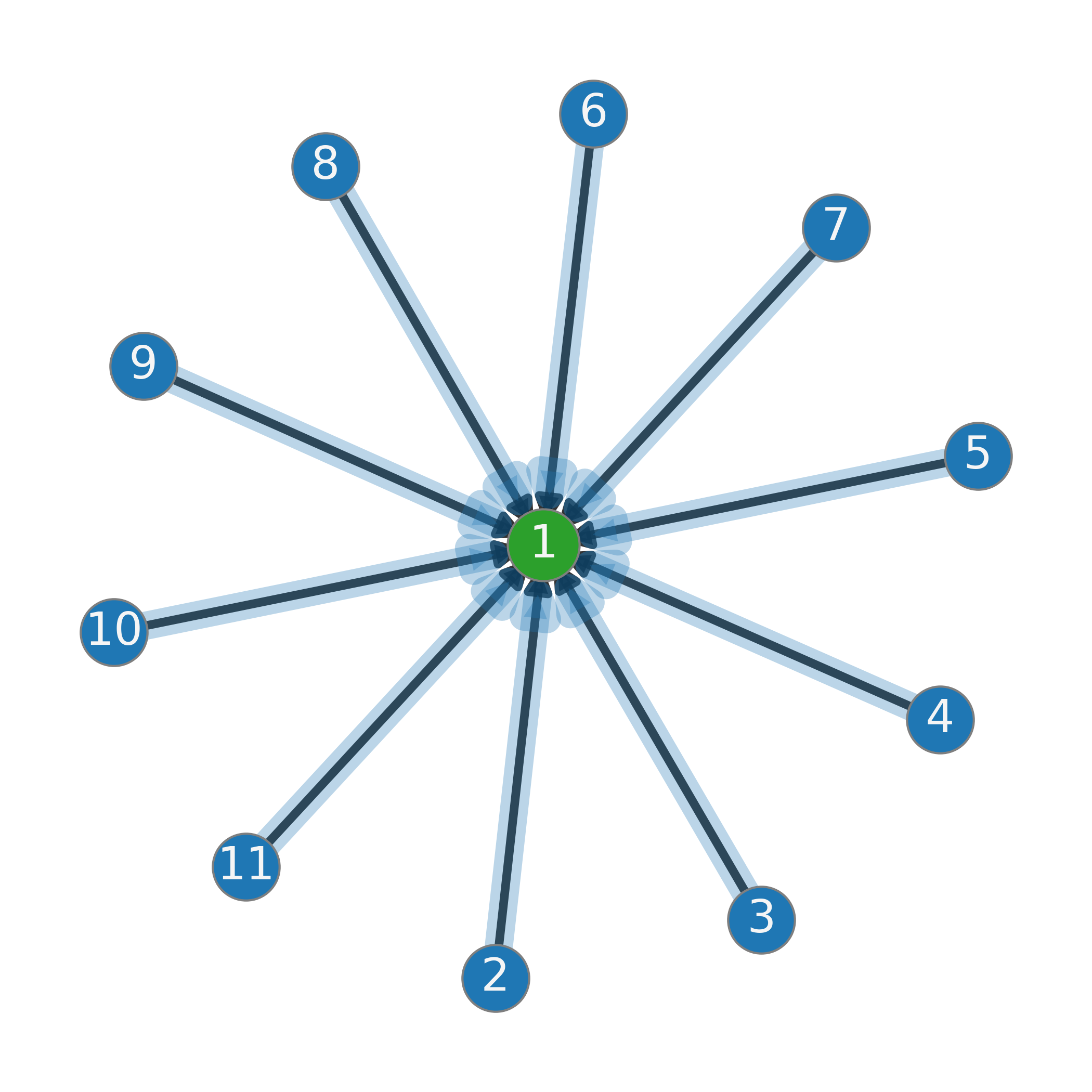}
    	\caption{}
    	\label{fig:in_star}
    \end{subfigure}
	\begin{subfigure}{0.495\textwidth}  
		\centering
		\includegraphics[scale=0.15]{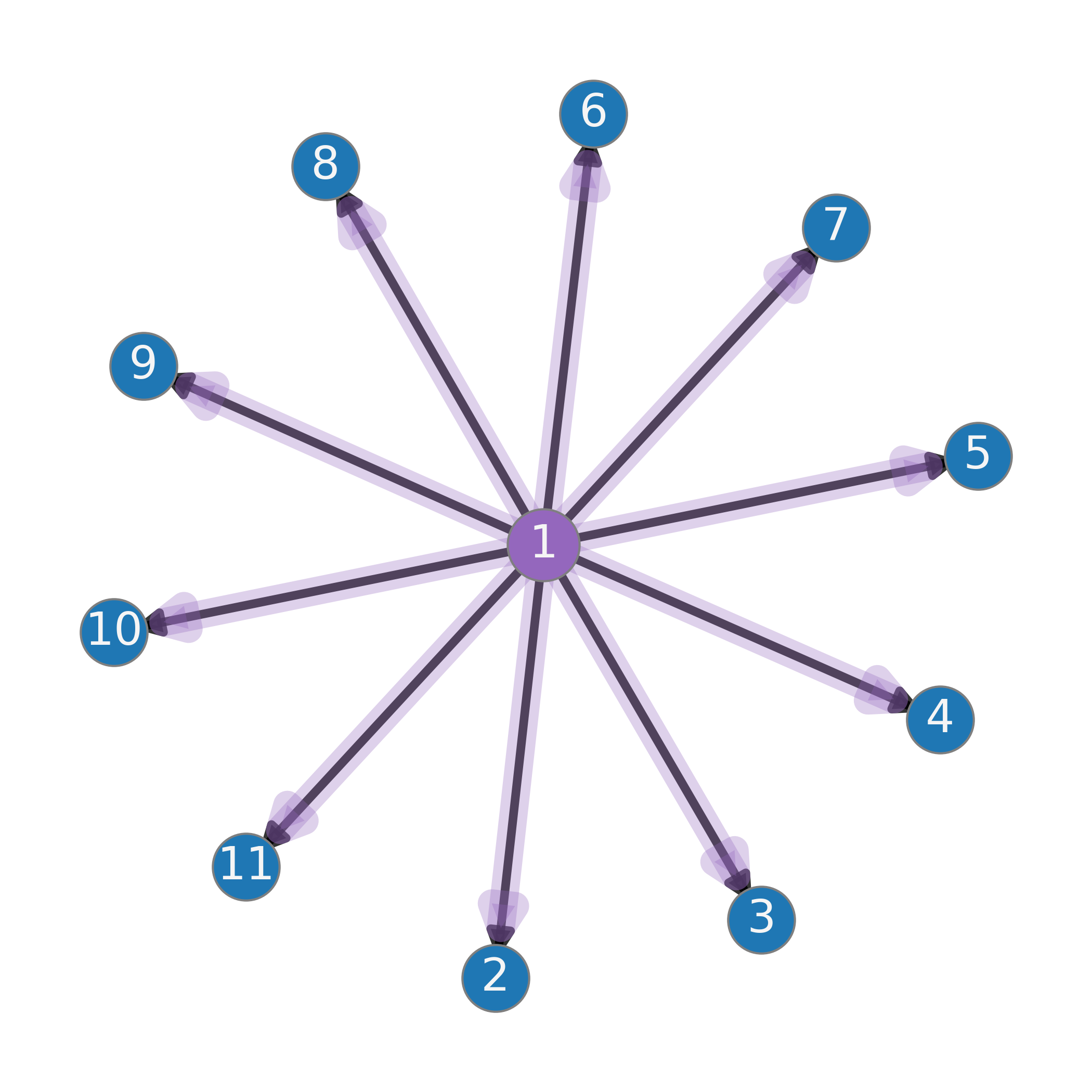}
		\caption{}
		\label{fig:out_star}
	\end{subfigure}
    \caption{Directed stars (a) $G^{(\text{in})}$, and (b) $G^{(\text{out})}$}
    \label{fig:star_graph}
\end{figure}

\begin{example}\label{ex:star}
    Let $G^{(\text{in})}$ and $G^{(\text{out})}$ be the directed star graphs with vertices $V = \{1, \ldots, N\}$ and edges pointing in/out to the central vertex as shown in Figure \ref{fig:star_graph}. Then the eigenvalues of $\bL_U^{(q,\text{in})}$, the unnormalized magnetic Laplacian on $G^{\text{in}}$, are given by 
    \begin{equation*}
    \lambda_1^\text{in}=0,\quad  \lambda_k^\text{in}=\frac{1}{2} \:\:\text{for }\: 2\leq k\leq N-1, \quad \text{ and }\quad\lambda_N^{\text{in}}=\frac{N}{2}.
    \end{equation*}
    If we let $v = 1$ be the central vertex, then the lead eigenvector is given by
    \begin{equation*}
        \mathbf{u}^\text{in}_1(1)=e^{2\pi i q},\quad \mathbf{u}^\text{in}_1(n)=1,\: 2\leq n\leq N.
    \end{equation*}
    For $2\leq k\leq N-1$, the eigenvectors are
    \begin{equation*}
        \mathbf{u}^\text{in}_k=\bm{\delta}_k-\bm{\delta}_{k+1},
    \end{equation*}
and the final eigenvector is given by 
\begin{equation*}\mathbf{u}^\text{in}_N(1)=-e^{2\pi i q},\quad
\mathbf{u}^\text{in}_N(n)=\frac{1}{N-1}, \:2\leq n \leq N.
\end{equation*}    The phase matrices satisfies $\bTheta^{(q,\text{in})}=-\bTheta^{(q,\text{out})}$. Therefore, the associated magnetic Laplacians satisfy $\bL_U^{(q,\text{in})}(u,v)=\overline{\bL_U^{(q,\text{out})}(u,v)}.$ Since these matrices are Hermitian, this implies that the corresponding eigenvalue-eigenvector pairs 
    satisfy $\lambda_k^{\text{in}}=\lambda_k^{\text{out}}$, and $\mathbf{u}_k^{\text{in}}=\overline{\mathbf{u}_k^{\text{out}}}$. Hence, $\mathbf{u}^\text{in}_1$ and $\mathbf{u}_1^{\text{out}}$ identify the central vertex, and  the sign of their imaginary parts at this vertex  identifies whether it is a source or a sink. On the other hand, the eigenvalues give no directional information.
    \end{example}

\begin{example}\label{ex:cycle}
     Let $G$ be the directed cycle. 
     Then, then the eigenvalues of  $\bL^{(q)}_U$ is are the classical Fourier modes $\mathbf{u}_k(n)= e^{(2\pi i k n/N)}$, independent of $q$. The eigenvalues, however, do depend on $q$ and are given by 
    \begin{equation*}
         \lambda_k=1-\cos\left(2\pi\left(\frac{k}{N}+q\right)\right),\quad 1\leq k \leq N.
    \end{equation*}
 \end{example}

\section{Expanded details of numerical results}\label{sec:tables}

Here we present more details on our node classification results in Tables \ref{table:class_alpha}, \ref{table:class_meta}, \ref{table:classcyc}, \ref{table:classcycfill}, and \ref{table:sup_class}; and more details of our link prediction results in Tables \ref{table:link_pred_sup}, \ref{table:link_sup_three2}, \ref{table:existence_sup}, \ref{table:direction_sup}, \ref{table:three_sup}, and \ref{table:threedirect_sup}. We present our results in the form mean $\pm$ standard deviation. 

The networks GCN, APPNP, GAT, SAGE, and GIN were not designed with directed graphs as the primary motivation. Therefore, we implemented these methods in two ways: (i) with the original asymmetric adjacency matrix; and (ii) with a symmetrized adjacency matrix. For node 
classification, symmetrizing the adjacency matrix improved performance for most of these networks on most of the real datasets. We did not test the symmetric implementations on our synthetic DSBM datasets because these datasets, by design, place a heavy importance on directional information.  For link prediction, on the other hand, we only use asymmetric adjacency matrices. 
In our tables below, GCN, APPNP, GAT, SAGE, and GIN refer to the implementations with the symmetrized adjacency matrix and GCN-D, APPNP-D, GAT-D, SAGE-D, and GIN-D refer to our implementation with the asymmetric matrix.

Tables \ref{table:class_alpha}, \ref{table:class_meta}, \ref{table:classcyc}, and \ref{table:classcycfill} provide the precise node classification results for the four types of DSBM graphs introduced in Section \ref{sec: dsbm} of the main text; they correspond to the plots in Figure \ref{fig:synthetic_plots} of the main text. Table \ref{table:sup_class} contains all of the information contained in Table \ref{table:class} from the main text, but reports separately the results of  GCN, APPNP, GAT, SAGE, and GIN (which use the symmetrized adjacency matrix) and the results of GCN-D, APPNP-D, GAT-D, SAGE-D, and GIN-D (which use the asymmetric adjacency matrix), whereas Table \ref{table:class} in the main text reported only the best-performer between the two variants.

With respect to link prediction, there are many results here in the appendix in addition to what is reported in the main text. Table \ref{table:link_pred_sup} is the same as Table \ref{table:link_pred} from the main text, except here in the  appendix we also include the \textit{Texas} data set. Table \ref{table:link_sup_three2} expands upon Table \ref{table:link_pred_sup} by considering the more difficult three-class classification problems described in Appendix \ref{sec: link prediction}. All of the results in these tables include undirected edges and, if present, multi-edges, which have essentially random labels with respect to their directionality (see also Section \ref{sec: link prediction}), and hence these results indicate the model's ability to ignore these noisy labels. MagNet performs quite well across this slate of link prediction experiments (top performer in 15/20 experiments). 

Tables \ref{table:existence_sup}, \ref{table:direction_sup}, \ref{table:three_sup}, and \ref{table:threedirect_sup} evaluate the same four link prediction tasks as Tables  \ref{table:link_pred_sup} and \ref{table:link_sup_three2}, except that undirected edges and multi-edges are not included in the training/validation/testing sets. Thus all labels are well-defined and noiseless. In this setting MagNet also performs very well, obtaining the top performance in 22/32 experiments across all four tables.

Aside from Table \ref{table:three_sup}, MagNet achieves the highest testing accuracy in 20/24 experiments. Digraph achieves the highest testing accuracy in 3/8 experiments, and MagNet is best in 2/8 experiments as shown in Table \ref{table:three_sup}. Having said that, there is not a statistically significant difference between MagNet and the top performing method in two other datasets (\textit{Wisconsin} and \textit{Cora-ML}), and MagNet is also a very close second on \textit{WikiCS}. Thus, MagNet is either the top performer or on par with the top performing method in 5/8 datasets in Table \ref{table:three_sup}. Nevertheless, the task is more difficult for MagNet than other tasks. We hypothesize that this is because half of the task is identifying whether there is an edge between $u,v$, or not; the other half, if there is an edge, is determining its direction. The first half of the task is an undirected task, and thus $q>0$ could provide noisy features for those pairs of vertices for which there is no edge. The Digraph method utilizes the symmetric Laplacian, which is unsuitable for direction prediction but works well for predicting the presence of an edge in either direction or the absence of an edge. The direction of the edge is more important in Tables \ref{table:existence_sup}, \ref{table:direction_sup}, and \ref{table:threedirect_sup}, and MagNet captures the direction information very well. The results indicate there is a trade-off between capturing undirected and directed features. This observation also leads to a potential future research direction that utilizes magnetic Laplacian matrices based on multiple values of $q$, making MagNet capture both undirected and directed information precisely.

\begin{table*}[h!]
    \caption{Node classification accuracy of ordered DSBM graphs with varying edge density. }
    \centerline{
    \begin{tabular}{c c c c} 
        \toprule
Method / $\alpha^{\ast} $ &0.1&0.08&0.05\\ \midrule
ChebNet&19.9$\pm$0.6&20.0$\pm$0.7&20.0$\pm$0.7\\
GCN-D&68.9$\pm$2.1&67.6$\pm$2.7&58.5$\pm$2.0\\\midrule
APPNP-D&97.7$\pm$1.7&95.9$\pm$2.2&\uline{90.3$\pm$2.4}\\
SAGE-D&20.1$\pm$1.1&19.9$\pm$0.8&19.9$\pm$1.0\\
GIN-D&57.3$\pm$5.8&55.4$\pm$5.5&50.9$\pm$7.7\\
GAT-D&42.1$\pm$5.3&39.0$\pm$7.0&37.2$\pm$5.5\\\midrule
DGCN&84.9$\pm$7.2&81.2$\pm$8.2&64.4$\pm$12.4\\
DiGraph&82.1$\pm$1.7&77.7$\pm$1.6&66.1$\pm$2.4\\
DiGraphIB&\uline{99.2$\pm$0.5}&\uline{97.7$\pm$0.7}&89.3$\pm$1.7\\\midrule
MagNet&\textbf{99.6$\pm$0.2}&\textbf{98.3$\pm$0.8}&\textbf{94.1$\pm$1.2}\\
\cmidrule{2-4}
Best $q$ & 0.25 & 0.10 & 0.25\\
\bottomrule
    \end{tabular}
    }
    \label{table:class_alpha}
\end{table*}

\begin{table*}[h!]
    \caption{Node classification accuracy of ordered DSBM graphs with varying net flow. }
    \setlength\tabcolsep{2.5 pt}
    \centerline{
    \begin{tabular}{c c c c c c c c c} 
        \toprule
        Method / $\beta^{\ast} $ &.05&.10&.15&.20&.25&.30&.35&.40\\
        \midrule
ChebNet&19.9$\pm$0.7&20.1$\pm$0.6&20.0$\pm$0.6&20.1$\pm$0.8&19.9$\pm$0.9&20.0$\pm$0.5&19.7$\pm$0.9&20.0$\pm$0.5\\
GCN-D&68.6$\pm$2.2&74.1$\pm$1.8&75.5$\pm$1.3&74.9$\pm$1.3&72.0$\pm$1.4&65.4$\pm$1.6&58.1$\pm$2.4&45.6$\pm$4.7\\\midrule
APPNP-D&97.4$\pm$1.8&94.3$\pm$2.4&89.4$\pm$3.6&79.8$\pm$9.0&69.4$\pm$3.9&59.6$\pm$4.9&51.8$\pm$4.5&39.4$\pm$5.3\\
SAGE-D&20.2$\pm$1.2&20.0$\pm$1.0&20.0$\pm$0.8&20.0$\pm$0.7&19.6$\pm$0.9&19.8$\pm$0.7&19.9$\pm$0.9&19.9$\pm$0.8\\
GIN-D&57.9$\pm$6.3&48.0$\pm$11.4&32.7$\pm$12.9&26.5$\pm$10.0&23.8$\pm$6.0&20.6$\pm$3.0&20.5$\pm$2.8&19.8$\pm$0.5\\
GAT-D&42.0$\pm$4.8&32.7$\pm$5.1&25.6$\pm$3.8&19.9$\pm$1.4&20.0$\pm$1.0&19.8$\pm$0.8&19.6$\pm$0.2&19.5$\pm$0.2\\\midrule
DGCN&81.4$\pm$1.1&84.7$\pm$0.7&85.5$\pm$1.0&86.2$\pm$0.8&84.2$\pm$1.1&\uline{78.4$\pm$1.3}&\textbf{69.6$\pm$1.5}&\textbf{54.3$\pm$1.5}\\
DiGraph&82.5$\pm$1.4&82.9$\pm$1.9&81.9$\pm$1.1&79.7$\pm$1.3&73.5$\pm$1.9&67.4$\pm$2.8&57.8$\pm$1.6&43.0$\pm$7.1\\
DiGraphIB&\uline{99.2$\pm$0.4}&\uline{97.9$\pm$0.6}&\uline{94.1$\pm$1.7}&\uline{88.7$\pm$2.0}&\uline{82.3$\pm$2.7}&70.0$\pm$2.2&57.8$\pm$6.4&41.0$\pm$9.0\\\midrule
MagNet&\textbf{99.6$\pm$0.2}&\textbf{99.0$\pm$1.0}&\textbf{97.5$\pm$0.8}&\textbf{94.2$\pm$1.6}&\textbf{88.7$\pm$1.9}&\textbf{79.4$\pm$2.9}&\uline{68.8$\pm$2.4}&\uline{51.8$\pm$3.1}\\
\cmidrule{2-9}
Best $q$ & 0.25 & 0.20 & 0.20 &0.25 & 0.20 & 0.20 &0.20 &0.25\\
\bottomrule
    \end{tabular}
    }
    \label{table:class_meta}
\end{table*}

\begin{table*}[h!]
    \caption{Node classification accuracy of cyclic DSBM graphs with varying net flow.}
    \centerline{
    \begin{tabular}{c c c c c c c } 
        \toprule
        Method / $\beta^{\ast} $ &.05&.10&.15&.20&.25&.30\\ \midrule
ChebNet&74.7$\pm$16.5&65.4$\pm$22.2&70.5$\pm$22.7&64.6$\pm$31.6&85.0$\pm$8.7&60.1$\pm$19.8\\
GCN-D&78.8$\pm$30.0&81.2$\pm$14.7&69.5$\pm$4.2&58.6$\pm$37.2&75.4$\pm$7.4&43.6$\pm$32.4\\\midrule
APPNP-D&19.6$\pm$0.5&19.5$\pm$0.4&19.6$\pm$0.3&19.6$\pm$0.3&19.6$\pm$0.4&19.6$\pm$0.3\\
SAGE-D&88.6$\pm$8.3&81.9$\pm$17.2&81.4$\pm$8.4&73.5$\pm$20.6&79.2$\pm$10.5&59.7$\pm$25.3\\
GIN-D&75.3$\pm$21.5&66.9$\pm$24.7&53.9$\pm$15.4&68.7$\pm$19.8&62.3$\pm$20.3&41.6$\pm$18.2\\
GAT-D&98.3$\pm$2.2&80.6$\pm$30.7&95.5$\pm$12.4&59.7$\pm$39.2&93.1$\pm$4.4&68.7$\pm$35.9\\\midrule
DGCN&83.7$\pm$23.1&\uline{99.8$\pm$0.2}&\textbf{99.4$\pm$0.8}&87.4$\pm$25.1&96.5$\pm$5.1&79.9$\pm$25.8\\
DiGraph&39.1$\pm$33.6&36.4$\pm$6.6&37.3$\pm$27.1&50.3$\pm$36.6&42.3$\pm$2.2&34.4$\pm$23.5\\
DiGraphIB&\uline{84.8$\pm$17.0}&94.2$\pm$6.6&\uline{99.2$\pm$0.6}&\textbf{98.1$\pm$1.1}&\uline{96.7$\pm$3.3}&\uline{84.7$\pm$7.4}\\\midrule
MagNet&\textbf{100.0$\pm$0.0}&\textbf{99.9$\pm$0.2}&87.4$\pm$28.4&\uline{96.8$\pm$12.5}&\textbf{100.0$\pm$0.0}&\textbf{99.4$\pm$0.6}\\
\cmidrule{2-7}
Best $q$ & 0.05 & 0.05 & 0.05 &0.05 & 0.1 & 0.1\\
\bottomrule
    \end{tabular}
    }
    \label{table:classcyc}
\end{table*}

\begin{table*}[h!]
    \caption{Node classification accuracy of noisy cyclic DSBM graphs with varying net flow. }
    \centerline{
    \begin{tabular}{c c c c c } 
        \toprule
Method / $\beta^{\ast} $ &.05&.10&.15&.20\\ \midrule
ChebNet&18.3$\pm$3.1&18.8$\pm$3.8&18.9$\pm$3.3&19.3$\pm$3.4\\
GCN-D&24.2$\pm$6.8&22.8$\pm$4.1&21.3$\pm$3.5&20.3$\pm$4.4\\\midrule
APPNP-D&17.4$\pm$1.8&17.9$\pm$2.0&17.8$\pm$1.8&17.6$\pm$2.6\\
SAGE-D&26.4$\pm$7.7&21.7$\pm$5.5&20.1$\pm$4.5&20.0$\pm$3.8\\
GIN-D&24.7$\pm$6.4&20.2$\pm$3.8&22.2$\pm$4.6&20.0$\pm$3.8\\
GAT-D&27.4$\pm$6.9&24.6$\pm$5.1&21.9$\pm$4.6&\uline{21.6$\pm$4.2}\\\midrule
DGCN&37.3$\pm$6.1&28.9$\pm$6.9&25.4$\pm$6.7&20.4$\pm$4.1\\
DiGraph&18.0$\pm$1.8&18.1$\pm$1.8&18.2$\pm$1.6&17.9$\pm$2.4\\
DiGraphIB&\uline{43.4$\pm$10.1}&\uline{32.3$\pm$10.1}&\uline{26.8$\pm$11.6}&19.1$\pm$2.9\\\midrule
MagNet&\textbf{80.5$\pm$7.0}&\textbf{63.7$\pm$8.2}&\textbf{56.9$\pm$6.7}&\textbf{70.2$\pm$5.1}\\
\cmidrule{2-5}
Best $q$ & 0.25 & 0.25 & 0.25 &0.20\\
\bottomrule
    \end{tabular}
    }
    \label{table:classcycfill}
\end{table*}

\begin{table*}
\centering
\caption{Testing accuracy of node classification. The best results are in \textbf{bold} and the second best are \uline{underlined}.}
\begin{tabular}{ c  H H H c c c  c c c} 
 \toprule
& Syn1 & Syn2 & Syn3 & Cornell & Texas & Wisconsin & Cora-ML & Citeseer &Telegram  \\ [0.5ex]
 \midrule
ChebNet&19.8$\pm$0.3&20.1$\pm$0.6&20.2$\pm$0.8&79.8$\pm$5.0&79.2$\pm$7.5&81.6$\pm$6.3&80.0$\pm$1.8&66.7$\pm$1.6&73.4 $\pm$5.8\\
GCN  &-&-&-&59.0$\pm$6.4&57.9$\pm$5.4&55.9$\pm$5.4&82.0$\pm$1.1&66.0$\pm$1.5&73.4$\pm$5.9\\
GCN-D&70.2$\pm$1.3&67.3$\pm$3.5&57.4$\pm$1.3&57.3$\pm$4.8&58.7$\pm$3.8&52.7$\pm$5.4&72.6$\pm$1.6&60.5$\pm$1.6&63.6$\pm$4.7\\ \midrule
APPNP&-&-&-&58.7$\pm$4.0&57.0$\pm$4.8&49.6$\pm$6.5&\textbf{82.6$\pm$1.4}&66.9$\pm$1.8& 69.4$\pm$3.5\\
APPNP-D&98.7$\pm$0.3&95.9$\pm$2.1&\uline{90.3$\pm$1.9}&58.4$\pm$3.0&56.8$\pm$2.7&51.8$\pm$7.4&68.6$\pm$2.5&58.6$\pm$1.8&66.4$\pm$5.0\\
GAT&-&-&-&57.6$\pm$4.9&61.1$\pm$5.0&54.1$\pm$4.2&81.9$\pm$1.0&\uline{67.3$\pm$1.3}&72.6$\pm$7.5\\
GAT-D&42.6$\pm$6.9&42.5$\pm$5.4&41.4$\pm$5.7&57.3$\pm$7.7&59.2$\pm$4.1&52.0$\pm$4.6&73.1$\pm$1.6&62.7$\pm$1.6&67.4$\pm$4.4\\
SAGE&-&-&-&77.6$\pm$6.3&\textbf{84.3$\pm$5.5}&79.2$\pm$5.3&\uline{82.3$\pm$1.2}&66.0$\pm$1.5&56.6$\pm$6.0\\
SAGE-D&19.5$\pm$0.9&20.1$\pm$0.8&19.7$\pm$1.1&\uline{80.0$\pm$6.1}&76.2$\pm$3.8&\uline{83.1$\pm$4.8}&72.0$\pm$2.1&61.8$\pm$2.0&55.0$\pm$7.4\\
GIN&-&-&-&57.9$\pm$5.7&65.2$\pm$6.5&58.2$\pm$5.1&78.1$\pm$2.0&63.3$\pm$2.5&74.4$\pm$8.1\\
GIN-D&54.9$\pm$2.6&54.5$\pm$4.3&55.2$\pm$3.5&55.4$\pm$5.2&58.1$\pm$5.3&50.2$\pm$7.6&67.0$\pm$3.2&60.4$\pm$2.3&68.8$\pm$4.0\\\midrule
DGCN&\uline{99.1$\pm$0.2}&\uline{97.3$\pm$0.3}&88.6$\pm$1.1&67.3$\pm$4.3&71.7$\pm$7.4&65.5$\pm$4.7&81.3$\pm$1.4&66.3$\pm$2.0&\textbf{90.4$\pm$5.6}\\
Digraph&80.7$\pm$2.2&78.1$\pm$0.8&66.5$\pm$1.1&66.8$\pm$6.2&64.9$\pm$8.1&59.6$\pm$3.8&79.4$\pm$1.8&62.6$\pm$2.2&82.0$\pm$3.1\\
DiGraphIB & & & & 64.4$\pm$9.0& 64.9$\pm$13.7 & 64.1$\pm$7.0 & 79.3$\pm$ 1.2 & 61.1$\pm$1.7 & 64.1$\pm$7.0\\\midrule
MagNet&\textbf{99.8$\pm$0.1}&\textbf{98.4$\pm$0.5}&\textbf{93.8$\pm$1.5}&\textbf{84.3$\pm$7.0}&\uline{83.3$\pm$6.1}&\textbf{85.7$\pm$3.2}&79.8$\pm$2.5&\textbf{67.5$\pm$1.8}&\uline{87.6 $\pm$2.9}\\ 
\cmidrule{2-10}
Best $q$ & x & x & x & 0.25 & 0.15 & 0.05 & 0.0 & 0.0 & 0.15\\
[0.5ex] \hline
\end{tabular}
\label{table:sup_class}
\end{table*}

\begin{landscape}
\begin{table}[t]
    \caption{Link prediction accuracy (\%) with noisy labels. The best results are in \textbf{bold} and the second best are \uline{underlined}.}
    \setlength\tabcolsep{2.8pt}
    \centerline{
    \begin{tabular}{c c c c c c c c c c c} 
        \toprule
&\multicolumn{5}{c}{Direction prediction} & \multicolumn{5}{c}{Existence prediction}\\
\cmidrule(lr){2-6} 
\cmidrule(lr){7-11}
&Cornell &Texas &Wisconsin &Cora-ML &CiteSeer&Cornell &Texas &Wisconsin &Cora-ML &CiteSeer \\\midrule
ChebNet&71.0$\pm$5.5&66.8$\pm$6.9&67.5$\pm$4.5&72.7$\pm$1.5&68.0$\pm$1.6&80.1$\pm$2.3&81.7$\pm$2.7&82.5$\pm$1.9&80.0$\pm$0.6&77.4$\pm$0.4\\
GCN&56.2$\pm$8.7&69.8$\pm$4.9&71.0$\pm$4.0&79.8$\pm$1.1&68.9$\pm$2.8&75.1$\pm$1.4&76.1$\pm$3.0&75.1$\pm$1.9&81.6$\pm$0.5&76.9$\pm$0.5\\\midrule
APPNP&69.5$\pm$9.0&76.8$\pm$5.1&75.1$\pm$3.5&\uline{83.7$\pm$0.7}&77.9$\pm$1.6&74.9$\pm$1.5&76.4$\pm$2.5&75.7$\pm$2.2&\uline{82.5$\pm$0.6}&78.6$\pm$0.7\\
SAGE&75.2$\pm$11.0&69.8$\pm$5.9&72.0$\pm$3.5&68.2$\pm$0.8&68.7$\pm$1.5&79.8$\pm$2.4&75.2$\pm$3.1&77.3$\pm$2.9&75.0$\pm$0.0&74.1$\pm$1.0\\
GIN&69.3$\pm$6.0&76.1$\pm$4.5&74.8$\pm$3.7&83.2$\pm$0.9&76.3$\pm$1.4&74.5$\pm$2.1&77.5$\pm$3.8&76.2$\pm$1.9&\uline{82.5$\pm$0.7}&77.9$\pm$0.7\\
GAT&67.9$\pm$11.1&50.0$\pm$2.0&53.2$\pm$2.6&50.0$\pm$0.1&50.6$\pm$0.5&77.9$\pm$3.2&74.9$\pm$0.3&74.6$\pm$0.0&75.0$\pm$0.0&75.0$\pm$0.0\\\midrule
DGCN&\textbf{80.7$\pm$6.3}&72.5$\pm$8.0&74.5$\pm$7.2&79.6$\pm$1.5&78.5$\pm$2.3&80.0$\pm$3.9&82.3$\pm$3.1&\uline{82.8$\pm$2.0}&82.1$\pm$0.5&\uline{81.2$\pm$0.4}\\
DiGraph&79.3$\pm$1.9&\uline{79.8$\pm$3.0}&\uline{82.3$\pm$4.9}&80.8$\pm$1.1&81.0$\pm$1.1&\textbf{80.6$\pm$2.5}&82.8$\pm$2.5&\uline{82.8$\pm$2.6}&81.8$\pm$0.5&\textbf{82.2$\pm$0.6}\\
DiGraphIB&79.8$\pm$4.8&\textbf{81.1$\pm$2.5}&82.0$\pm$4.9&83.4$\pm$1.1&\uline{82.5$\pm$1.3}&\uline{80.5$\pm$3.6}&\uline{83.6$\pm$2.6}&82.4$\pm$2.2&82.2$\pm$0.5&81.0$\pm$0.5\\\midrule
MagNet&\textbf{80.7$\pm$2.7}&79.5$\pm$3.7&\textbf{83.6$\pm$2.8}&\textbf{86.1$\pm$0.9}&\textbf{85.1$\pm$0.8}&\textbf{80.6$\pm$3.8}&\textbf{83.8$\pm$3.3}&\textbf{82.9$\pm$2.6}&\textbf{82.8$\pm$0.7}&79.9$\pm$0.5\\
\cmidrule{2-11}
Best $q$ & 0.10 & 0.10 & 0.05 & 0.05 & 0.15 & 0.25 & 0.10 & 0.25 & 0.05 & 0.05\\
\bottomrule
\end{tabular}
}
\label{table:link_pred_sup}
\end{table}

\begin{table}[t]
    \caption{Link prediction accuracy based on three classes labels(\%) with noisy labels. The best results are in \textbf{bold} and the second best are \uline{underlined}.}
    \setlength\tabcolsep{2.8pt}
    \centerline{
    \begin{tabular}{c c c c c c c c c c c} 
        \toprule
&\multicolumn{5}{c}{Three classes link prediction} & \multicolumn{5}{c}{Direction prediction by three classes training}\\
\cmidrule(lr){2-6} 
\cmidrule(lr){7-11}
&Cornell &Texas &Wisconsin &Cora-ML &CiteSeer&Cornell &Texas &Wisconsin &Cora-ML &CiteSeer \\\midrule
ChebNet&60.0$\pm$2.3&64.8$\pm$2.4&65.9$\pm$2.8&64.8$\pm$0.7&58.0$\pm$0.9&70.7$\pm$2.8&58.9$\pm$4.7&67.0$\pm$4.1&72.6$\pm$1.3&67.8$\pm$1.8\\
GCN&51.9$\pm$3.2&54.1$\pm$3.0&51.7$\pm$2.6&66.9$\pm$0.5&54.6$\pm$1.2&56.0$\pm$8.1&69.8$\pm$5.4&68.0$\pm$3.2&79.7$\pm$1.1&68.8$\pm$2.2\\
APPNP&55.4$\pm$4.9&56.7$\pm$2.7&53.7$\pm$3.2&67.8$\pm$0.7&57.8$\pm$1.2&76.0$\pm$6.6&76.4$\pm$4.0&72.5$\pm$3.9&82.5$\pm$0.7&78.2$\pm$1.7\\
SAGE&62.1$\pm$4.2&52.2$\pm$2.6&56.4$\pm$3.9&50.0$\pm$0.0&48.9$\pm$1.1&74.8$\pm$6.7&69.5$\pm$4.2&71.4$\pm$3.6&68.2$\pm$0.9&68.7$\pm$1.8\\
GIN&52.3$\pm$5.2&57.6$\pm$2.4&54.9$\pm$3.5&\textbf{68.0$\pm$0.8}&56.8$\pm$1.3&69.5$\pm$5.5&76.6$\pm$4.3&74.2$\pm$4.0&83.2$\pm$1.0&76.3$\pm$1.3\\
GAT&57.3$\pm$6.4&50.0$\pm$0.0&49.9$\pm$0.8&50.0$\pm$0.0&50.0$\pm$0.1&69.0$\pm$6.8&50.5$\pm$1.7&51.4$\pm$2.4&50.1$\pm$0.1&50.2$\pm$0.6\\
DGCN&63.2$\pm$4.7&64.9$\pm$2.8&65.9$\pm$2.7&67.2$\pm$0.6&\uline{63.7$\pm$0.7}&76.7$\pm$4.4&63.4$\pm$8.9&68.1$\pm$6.7&79.4$\pm$1.4&78.6$\pm$2.2\\
DiGraph&\uline{63.7$\pm$3.7}&66.6$\pm$3.3&\textbf{67.2$\pm$2.2}&66.4$\pm$0.6&\textbf{64.7$\pm$0.7}&79.5$\pm$1.9&\uline{80.0$\pm$3.8}&\uline{82.8$\pm$5.3}&80.7$\pm$1.0&79.9$\pm$1.2\\
DiGraphIB&62.0$\pm$5.3&\uline{67.4$\pm$2.5}&66.0$\pm$1.5&66.0$\pm$0.6&62.0$\pm$0.9&\uline{81.4$\pm$2.8}&79.8$\pm$3.1&\uline{82.8$\pm$5.2}&\uline{83.2$\pm$1.1}&\uline{82.5$\pm$1.6}\\
MagNet&\textbf{64.5$\pm$4.3}&\textbf{67.6$\pm$3.5}&\uline{66.5$\pm$2.4}&\uline{67.7$\pm$0.9}&60.4$\pm$0.9&\textbf{83.3$\pm$3.5}&\textbf{80.5$\pm$5.5}&\textbf{85.7$\pm$2.6}&\textbf{86.2$\pm$0.8}&\textbf{84.8$\pm$1.3}\\
\cmidrule{2-11}
Best $q$ & 0.25 & 0.15 & 0.25 & 0.05 & 0.05 & 0.25 & 0.15 & 0.25 & 0.05 & 0.05\\
[0.5ex]
\bottomrule
\end{tabular}
}
\label{table:link_sup_three2}
\end{table}
\end{landscape}

\begin{table}[t]
    \caption{Existence prediction(\%) with noiseless labels. The best results are in \textbf{bold} and the second best are \uline{underlined}.}
    \setlength\tabcolsep{2pt}
    \centerline{
    \begin{tabular}{c c c c c c c c c} 
        \toprule
&Cornell &Texas &Wisconsin &Cora-ML &CiteSeer &WikiCS &Chameleon &Squirrel \\\midrule
ChebNet&68.6$\pm$5.1&67.7$\pm$9.9&70.1$\pm$5.6&71.2$\pm$0.8&66.0$\pm$1.6&78.4$\pm$0.3&88.7$\pm$0.3&90.4$\pm$0.2\\
GCN-D&56.7$\pm$10.4&66.1$\pm$7.5&62.9$\pm$6.0&75.5$\pm$1.1&64.0$\pm$1.8&78.3$\pm$0.3&90.1$\pm$0.3&92.0$\pm$0.2\\\midrule
APPNP-D&65.2$\pm$9.0&72.1$\pm$6.9&71.5$\pm$4.0&\textbf{78.6$\pm$0.7}&71.0$\pm$0.8&80.6$\pm$0.3&\uline{90.4$\pm$0.4}&\textbf{91.8$\pm$0.2}\\
SAGE-D&71.2$\pm$7.7&66.6$\pm$7.2&70.5$\pm$5.5&70.1$\pm$1.4&64.0$\pm$1.6&62.2$\pm$0.3&86.1$\pm$0.6&83.7$\pm$0.2\\
GIN-D&63.8$\pm$7.1&72.1$\pm$5.7&70.1$\pm$3.6&\uline{78.3$\pm$1.0}&70.1$\pm$0.9&80.5$\pm$0.3&\uline{90.4$\pm$0.4}&92.1$\pm$0.1\\
GAT-D&62.6$\pm$9.9&50.0$\pm$1.8&50.9$\pm$1.6&50.0$\pm$0.1&50.2$\pm$0.5&50.2$\pm$0.3&50.1$\pm$0.2&58.8$\pm$13.4\\\midrule
DGCN&73.2$\pm$5.3&67.1$\pm$9.8&71.8$\pm$4.5&74.0$\pm$1.0&73.4$\pm$1.2&\uline{80.7$\pm$0.3}&89.1$\pm$0.4&\uline{91.5$\pm$0.2}\\
DiGraph&71.6$\pm$5.3&84.2$\pm$3.8&\uline{79.4$\pm$3.3}&75.7$\pm$1.1&74.0$\pm$1.3&76.8$\pm$0.3&89.3$\pm$0.4&91.4$\pm$0.1\\
DiGraphIB&\uline{73.4$\pm$4.4}&\uline{85.1$\pm$5.6}&77.9$\pm$3.8&76.0$\pm$1.0&\uline{74.3$\pm$2.0}&76.9$\pm$0.4&89.3$\pm$0.5&90.8$\pm$0.1\\\midrule
MagNet&\textbf{74.7$\pm$5.4}&\textbf{85.6$\pm$4.5}&\textbf{80.1$\pm$6.2}&77.9$\pm$1.0&\textbf{76.9$\pm$1.4}&\textbf{83.3$\pm$0.2}&\textbf{90.7$\pm$0.4}&\uline{91.5$\pm$0.2}\\
\cmidrule{2-9}
Best $q$ & 0.05 & 0.10 & 0.20 & 0.10 & 0.10 & 0.05 & 0.05 & 0.05\\
\bottomrule
\end{tabular}
}
\label{table:existence_sup}
\end{table}

\begin{table}[h]
    \caption{Direction prediction(\%) with noiseless labels. The best results are in \textbf{bold} and the second best  are \uline{underlined}.}
    \setlength\tabcolsep{2pt}
    \centerline{
    \begin{tabular}{c c c c c c c c c} 
        \toprule
&Cornell &Texas &Wisconsin &Cora-ML &CiteSeer &WikiCS &Chameleon &Squirrel \\\midrule
ChebNet&74.1$\pm$5.6&72.3$\pm$10.0&69.9$\pm$6.2&73.3$\pm$1.2&69.2$\pm$2.1&71.1$\pm$0.3&94.6$\pm$0.2&95.3$\pm$0.2\\
GCN-D&54.4$\pm$8.8&76.7$\pm$6.3&73.8$\pm$4.2&80.8$\pm$1.1&70.8$\pm$2.3&78.4$\pm$0.2&97.2$\pm$0.2&97.2$\pm$0.1\\\midrule
APPNP-D&73.6$\pm$6.6&83.6$\pm$4.3&80.8$\pm$4.5&\uline{85.6$\pm$0.8}&81.0$\pm$1.8&82.9$\pm$0.2&\uline{97.6$\pm$0.2}&98.1$\pm$0.1\\
SAGE-D&77.0$\pm$5.5&77.7$\pm$6.5&76.4$\pm$3.8&69.3$\pm$0.5&70.1$\pm$1.6&56.0$\pm$0.2&94.4$\pm$0.3&93.6$\pm$1.8\\
GIN-D&69.4$\pm$6.6&84.7$\pm$4.5&80.6$\pm$3.8&84.5$\pm$0.9&78.5$\pm$1.4&82.9$\pm$0.1&97.6$\pm$0.2&98.0$\pm$0.1\\
GAT-D&71.8$\pm$10.1&51.1$\pm$1.8&52.2$\pm$2.0&50.1$\pm$0.2&50.7$\pm$0.5&50.2$\pm$0.4&50.5$\pm$1.3&68.6$\pm$16.8\\\midrule
DGCN&82.9$\pm$5.9&80.8$\pm$10.8&76.8$\pm$8.8&80.3$\pm$1.5&81.6$\pm$2.0&81.6$\pm$0.3&96.6$\pm$0.2&\uline{98.0$\pm$0.1}\\
DiGraph&83.1$\pm$4.9&89.0$\pm$2.8&\uline{87.8$\pm$4.1}&82.0$\pm$1.0&84.0$\pm$1.5&79.6$\pm$0.2&97.1$\pm$0.2&96.9$\pm$0.1\\
DiGraphIB&\uline{83.7$\pm$5.6}&\uline{89.5$\pm$3.3}&\uline{87.8$\pm$3.9}&84.3$\pm$1.4&\uline{85.1$\pm$1.4}&\uline{83.0$\pm$0.2}&\uline{97.6$\pm$0.2}&97.2$\pm$0.1\\
MagNet&\textbf{88.8$\pm$4.9}&\textbf{91.9$\pm$4.5}&\textbf{89.3$\pm$4.2}&\textbf{87.3$\pm$0.8}&\textbf{88.1$\pm$0.9}&\textbf{86.2$\pm$0.2}&\textbf{97.8$\pm$0.2}&\textbf{98.1$\pm$0.1}\\
\cmidrule{2-9}
Best $q$ & 0.10 & 0.05 & 0.25 & 0.20 & 0.10 & 0.05&0.20 & 0.10\\
\bottomrule
\end{tabular}
}
\label{table:direction_sup}
\end{table}

\begin{table}[h]
    \caption{Three classes link prediction(\%) with noiseless labels. The best results are in \textbf{bold} and the second best  are \uline{underlined}.}
    \setlength\tabcolsep{2pt}
    \centerline{
    \begin{tabular}{c c c c c c c c c} 
        \toprule
&Cornell &Texas &Wisconsin &Cora-ML &CiteSeer &WikiCS &Chameleon &Squirrel \\\midrule
ChebNet&63.0$\pm$2.1&71.5$\pm$2.0&70.5$\pm$2.1&65.6$\pm$0.5&60.3$\pm$0.8&74.3$\pm$0.1&80.7$\pm$0.3&83.9$\pm$0.1\\
GCN-D&53.0$\pm$2.5&62.6$\pm$2.6&55.8$\pm$3.1&67.5$\pm$0.6&57.1$\pm$1.1&74.5$\pm$0.2&81.3$\pm$0.3&86.2$\pm$0.1\\
APPNP-D&61.5$\pm$3.6&63.3$\pm$3.3&57.9$\pm$4.3&\textbf{68.6$\pm$0.7}&60.3$\pm$1.2&75.8$\pm$0.2&81.2$\pm$0.2&85.9$\pm$0.1\\
SAGE-D&64.8$\pm$4.0&59.2$\pm$3.2&60.7$\pm$4.6&50.9$\pm$0.1&51.1$\pm$1.2&60.8$\pm$0.1&70.0$\pm$0.3&67.3$\pm$0.2\\
GIN-D&54.6$\pm$3.8&65.2$\pm$3.7&58.4$\pm$4.0&\textbf{68.6$\pm$0.7}&59.1$\pm$1.4&76.2$\pm$0.2&81.7$\pm$0.3&\textbf{86.5$\pm$0.3}\\
GAT-D&58.8$\pm$6.4&56.9$\pm$1.4&54.2$\pm$1.0&50.8$\pm$0.1&52.2$\pm$0.2&60.8$\pm$0.1&54.2$\pm$0.1&52.5$\pm$0.0\\\midrule
DGCN&65.1$\pm$6.1&73.6$\pm$3.6&71.6$\pm$1.7&67.9$\pm$0.5&\uline{66.0$\pm$0.7}&\textbf{77.6$\pm$0.1}&80.9$\pm$0.3&85.4$\pm$0.1\\
DiGraph&\uline{66.1$\pm$4.7}&\uline{76.4$\pm$4.0}&\textbf{72.9$\pm$2.0}&67.2$\pm$0.7&\textbf{67.5$\pm$0.6}&74.4$\pm$0.2&\textbf{83.8$\pm$0.3}&\uline{86.4$\pm$0.2}\\
DiGraphIB&64.5$\pm$4.1&76.2$\pm$4.3&\uline{72.4$\pm$2.6}&66.6$\pm$0.5&64.4$\pm$0.6&71.8$\pm$0.2&\uline{83.4$\pm$0.2}&85.6$\pm$0.1\\\midrule
MagNet&\textbf{66.4$\pm$5.0}&\textbf{76.6$\pm$3.9}&71.3$\pm$2.3&68.4$\pm$0.8&63.2$\pm$1.3&\uline{77.0$\pm$0.2}&81.9$\pm$0.4&84.8$\pm$0.1\\
\cmidrule{2-9}
Best $q$ & 0.25 & 0.05 & 0.25 & 0.05 & 0.05 & 0.05 & 0.05 & 0.10\\
\bottomrule
\end{tabular}
}
\label{table:three_sup}
\end{table}

\begin{table}[h]
    \caption{Direction prediction by three classes link prediction(\%) with noiseless labels. The best results are in \textbf{bold} and the second best  are \uline{underlined}.}
    \setlength\tabcolsep{2pt}
    \centerline{
    \begin{tabular}{c c c c c c c c c} 
        \toprule
&Cornell &Texas &Wisconsin &Cora-ML &CiteSeer  &WikiCS &Chameleon &Squirrel\\\midrule
Cheb&75.6$\pm$4.9&61.6$\pm$5.4&69.7$\pm$4.1&73.4$\pm$1.3&69.4$\pm$1.5&71.1$\pm$0.2&94.6$\pm$0.2&95.3$\pm$0.1\\
GCN-D&56.6$\pm$3.0&77.9$\pm$7.0&70.9$\pm$5.1&80.6$\pm$1.1&70.3$\pm$2.1&78.4$\pm$0.2&97.2$\pm$0.2&97.2$\pm$0.1\\\midrule
APPNP-D&75.5$\pm$4.5&83.5$\pm$4.4&79.9$\pm$3.4&83.6$\pm$0.8&80.7$\pm$1.4&82.7$\pm$0.2&\uline{97.5$\pm$0.2}&\uline{98.0$\pm$0.1}\\
SAGE-D&77.3$\pm$4.5&75.0$\pm$5.4&75.8$\pm$5.0&69.2$\pm$0.6&69.7$\pm$1.6&56.0$\pm$0.3&94.4$\pm$0.3&92.8$\pm$1.3\\
GIN-D&71.9$\pm$4.6&85.6$\pm$4.1&80.6$\pm$3.8&\uline{84.4$\pm$0.8}&78.6$\pm$2.0&\uline{82.9$\pm$0.2}&97.6$\pm$0.2&\textbf{98.1$\pm$0.1}\\
GAT-D&67.3$\pm$10.8&49.7$\pm$1.6&52.3$\pm$1.9&50.0$\pm$0.1&50.1$\pm$0.3&50.2$\pm$0.5&50.1$\pm$0.1&50.0$\pm$0.0\\\midrule
DGCN&79.9$\pm$6.1&68.0$\pm$8.9&77.0$\pm$4.3&80.1$\pm$1.1&81.1$\pm$2.6&81.6$\pm$0.3&96.4$\pm$0.2&\uline{98.0$\pm$0.1}\\
DiGraph&\uline{85.5$\pm$4.1}&\textbf{90.4$\pm$4.0}&\uline{87.6$\pm$4.7}&82.0$\pm$1.0&83.0$\pm$1.2&79.6$\pm$0.2&97.1$\pm$0.2&96.9$\pm$0.1\\
DiGraphIB&85.2$\pm$4.7&\uline{89.9$\pm$3.4}&87.5$\pm$4.2&84.2$\pm$1.1&\uline{85.2$\pm$1.3}&82.2$\pm$0.3&97.1$\pm$0.2&96.9$\pm$0.1\\\midrule
MagNet&\textbf{88.5$\pm$3.9}&87.9$\pm$8.4&\textbf{90.2$\pm$3.1}&\textbf{87.4$\pm$0.7}&\textbf{87.9$\pm$1.1}&\textbf{85.4$\pm$0.3}&\textbf{97.8$\pm$0.2}&\uline{98.0$\pm$0.1}\\
\cmidrule{2-9}
Best $q$ & 0.25 & 0.05 & 0.25 & 0.05 & 0.05  & 0.05 & 0.05 & 0.10\\
\bottomrule
\end{tabular}
}
\label{table:threedirect_sup}
\end{table}

\section{Optimal $q$ values for synthetic data}\label{sec: bestq}
Optimal $q$ values for synthetic graphs are shown in Tables \ref{table:class_alpha}, \ref{table:class_meta}, \ref{table:classcyc}, and \ref{table:classcycfill}. 
We observe that the optimal $q$ is smaller for node classification of cyclic DSBM graphs than the ordered and noisy cyclic DSBM graphs.
For cyclic DSBM graphs, the cluster is relatively clear by checking connectivity even without direction information.
But the direction is crucial for classification for the other two types of DSBM graphs.
It indicates that a smaller $q$ $(q<0.15)$ is enough for node classification of directed graphs when the direction is less critical. And a larger $q$ $(q>0.15)$ is needed to encode more direction information in the phase matrix for better performance. If the cluster is evident in the symmetrized adjacency matrix, we can use $q=0$, and MagNet will reduce to ChebNet as in the results on \textit{Cora-ML} and \textit{CiteSeer} in Table \ref{table:sup_class}.

\FloatBarrier

\bibliographystyle{plain}
\bibliography{dgnnbibliography_neurips}

\begin{thebibliography}{10}

\bibitem{atwood2015diffusion}
James Atwood and Don Towsley.
\newblock Diffusion-convolutional neural networks.
\newblock In D.~Lee, M.~Sugiyama, U.~Luxburg, I.~Guyon, and R.~Garnett,
  editors, {\em Advances in Neural Information Processing Systems}, volume~29,
  pages 1993--2001. Curran Associates, Inc., 2016.

\bibitem{belkin2003laplacian}
Mikhail Belkin and Partha Niyogi.
\newblock Laplacian eigenmaps for dimensionality reduction and data
  representation.
\newblock {\em Neural computation}, 15(6):1373--1396, 2003.

\bibitem{benson2016higher}
Austin~R Benson, David~F Gleich, and Jure Leskovec.
\newblock Higher-order organization of complex networks.
\newblock {\em Science}, 353(6295):163--166, 2016.

\bibitem{bojchevski2017deep}
Aleksandar Bojchevski and Stephan G{\"u}nnemann.
\newblock Deep gaussian embedding of graphs: Unsupervised inductive learning
  via ranking.
\newblock In {\em ICLR Workshop on Representation Learning on Graphs and
  Manifolds}, 2017.

\bibitem{bovet2020activity}
Alexandre Bovet and Peter Grindrod.
\newblock The activity of the far right on telegram.
\newblock
  \url{https://www.researchgate.net/publication/346968575_The_Activity_of_the_Far_Right_on_Telegram_v21},
  2020.

\bibitem{bruna:spectralNN2014}
Joan Bruna, Wojciech Zaremba, Arthur Szlam, and Yann LeCun.
\newblock Spectral networks and deep locally connected networks on graphs.
\newblock In {\em International Conference on Learning Representations (ICLR)},
  2014.

\bibitem{chung2005Laplacians}
Fan Chung.
\newblock Laplacians and the cheeger inequality for directed graphs.
\newblock {\em Annals of Combinatorics}, 9(1):1--19, 2005.

\bibitem{chung2013local}
Fan Chung and Mark Kempton.
\newblock A local clustering algorithm for connection graphs.
\newblock In {\em International Workshop on Algorithms and Models for the
  Web-Graph}, pages 26--43. Springer, 2013.

\bibitem{chung1997spectral}
Fan~RK Chung and Fan~Chung Graham.
\newblock {\em Spectral graph theory}.
\newblock Number~92. American Mathematical Soc., 1997.

\bibitem{CLONINGER2017370}
Alexander Cloninger.
\newblock A note on markov normalized magnetic eigenmaps.
\newblock {\em Applied and Computational Harmonic Analysis}, 43(2):370 -- 380,
  2017.

\bibitem{coifman2006diffusion}
Ronald~R Coifman and St{\'e}phane Lafon.
\newblock Diffusion maps.
\newblock {\em Applied and computational harmonic analysis}, 21(1):5--30, 2006.

\bibitem{cucuringu2020hermitian}
Mihai Cucuringu, Huan Li, He~Sun, and Luca Zanetti.
\newblock Hermitian matrices for clustering directed graphs: insights and
  applications.
\newblock In {\em International Conference on Artificial Intelligence and
  Statistics}, pages 983--992. PMLR, 2020.

\bibitem{Defferrard2018}
Micha\"{e}l Defferrard, Xavier Bresson, and Pierre Vandergheynst.
\newblock Convolutional neural networks on graphs with fast localized spectral
  filtering.
\newblock In {\em Advances in Neural Information Processing Systems 29}, pages
  3844--3852, 2016.

\bibitem{duvenaud2015convolutional}
David~K Duvenaud, Dougal Maclaurin, Jorge Iparraguirre, Rafael Bombarell,
  Timothy Hirzel, Alan Aspuru-Guzik, and Ryan~P Adams.
\newblock Convolutional networks on graphs for learning molecular fingerprints.
\newblock In C.~Cortes, N.~Lawrence, D.~Lee, M.~Sugiyama, and R.~Garnett,
  editors, {\em Advances in Neural Information Processing Systems}, volume~28,
  pages 2224--2232. Curran Associates, Inc., 2015.

\bibitem{f2020characterization}
Bruno~Messias F.~de Resende and Luciano~da F.~Costa.
\newblock Characterization and comparison of large directed networks through
  the spectra of the magnetic laplacian.
\newblock {\em Chaos: An Interdisciplinary Journal of Nonlinear Science},
  30(7):073141, 2020.

\bibitem{fanuel:magneticEigenmaps2018}
Micha\"{e}l Fanuel, Carlos~M. Ala\'{i}z, \'{A}ngela Fern\'{a}ndez, and
  Johan~A.K. Suykens.
\newblock Magnetic eigenmaps for the visualization of directed networks.
\newblock {\em Applied and Computational Harmonic Analysis}, 44:189--199, 2018.

\bibitem{fanuel2017Magnetic}
Micha{\"e}l Fanuel, Carlos~M Alaiz, and Johan~AK Suykens.
\newblock Magnetic eigenmaps for community detection in directed networks.
\newblock {\em Physical Review E}, 95(2):022302, 2017.

\bibitem{furutani:GSPdirectedHermit2019}
Satoshi Furutani, Toshiki Shibahara, Mitsuaki Akiyama, Kunio Hato, and Masaki
  Aida.
\newblock Graph signal processing for directed graphs based on the hermitian
  laplacian.
\newblock In {\em Machine Learning and Knowledge Discovery in Databases}, pages
  447--463, 2020.

\bibitem{krystal:hermitianAdjDigraphs2017}
Krystal Guo and Bojan Mohar.
\newblock Hermitian adjacency matrix of digraphs and mixed graphs.
\newblock {\em Journal of Graph Theory}, 85(1):217--248, 2017.

\bibitem{hamilton2017inductive}
William~L. Hamilton, Rex Ying, and Jure Leskovec.
\newblock Inductive representation learning on large graphs.
\newblock In {\em Proceedings of the 31st International Conference on Neural
  Information Processing Systems}, NIPS'17, page 1025–1035, Red Hook, NY,
  USA, 2017. Curran Associates Inc.

\bibitem{hammond2011wavelets}
David~K Hammond, Pierre Vandergheynst, and R{\'e}mi Gribonval.
\newblock Wavelets on graphs via spectral graph theory.
\newblock {\em Applied and Computational Harmonic Analysis}, 30(2):129--150,
  2011.

\bibitem{kipf2016semi}
Thomas~N. Kipf and Max Welling.
\newblock Semi-supervised classification with graph convolutional networks.
\newblock In {\em International Conference on Learning Representations (ICLR)},
  2017.

\bibitem{klicpera2018predict}
Johannes Klicpera, Aleksandar Bojchevski, and Stephan G{\"u}nnemann.
\newblock Predict then propagate: Graph neural networks meet personalized
  pagerank.
\newblock In {\em ICLR}, 2019.

\bibitem{levie2019transferability}
Ron Levie, Wei Huang, Lorenzo Bucci, Michael~M Bronstein, and Gitta Kutyniok.
\newblock Transferability of spectral graph convolutional neural networks.
\newblock {\em arXiv preprint arXiv:1907.12972}, 2019.

\bibitem{lieb1993fluxes}
Elliott~H Lieb and Michael Loss.
\newblock Fluxes, laplacians, and kasteleyn’s theorem.
\newblock In {\em Statistical Mechanics}, pages 457--483. Springer, 1993.

\bibitem{ma:spectralDGCN2019}
Yi~Ma, Jianye Hao, Yaodong Yang, Han Li, Junqi Jin, and Guangyong Chen.
\newblock Spectral-based graph convolutional network for directed graphs.
\newblock arXiv:1907.08990, 2019.

\bibitem{Marques2020}
Antonio~G Marques, Santiago Segarra, and Gonzalo Mateos.
\newblock Signal processing on directed graphs: The role of edge directionality
  when processing and learning from network data.
\newblock {\em IEEE Signal Processing Magazine}, 37(6):99--116, 2020.

\bibitem{mernyei2020wiki}
P{\'e}ter Mernyei and C{\u{a}}t{\u{a}}lina Cangea.
\newblock Wiki-cs: A wikipedia-based benchmark for graph neural networks.
\newblock {\em arXiv preprint arXiv:2007.02901}, 2020.

\bibitem{mohar2020new}
Bojan Mohar.
\newblock A new kind of hermitian matrices for digraphs.
\newblock {\em Linear Algebra and its Applications}, 584:343--352, 2020.

\bibitem{monti:MotifNet2018}
Federico Monti, Karl Otness, and Michael~M. Bronstein.
\newblock Motifnet: A motif-based graph convolutional network for directed
  graphs.
\newblock In {\em 2018 IEEE Data Science Workshop}, pages 225--228, 2018.

\bibitem{ortega2018graph}
Antonio Ortega, Pascal Frossard, Jelena Kova{\v{c}}evi{\'c}, Jos{\'e}~MF Moura,
  and Pierre Vandergheynst.
\newblock Graph signal processing: Overview, challenges, and applications.
\newblock {\em Proceedings of the IEEE}, 106(5):808--828, 2018.

\bibitem{pei2020geom}
Hongbin Pei, Bingzhe Wei, Kevin Chen-Chuan Chang, Yu~Lei, and Bo~Yang.
\newblock Geom-gcn: Geometric graph convolutional networks.
\newblock {\em arXiv preprint arXiv:2002.05287}, 2020.

\bibitem{rozemberczki2019multiscale}
Benedek Rozemberczki, Carl Allen, and Rik Sarkar.
\newblock Multi-scale attributed node embedding.
\newblock {\em arXiv preprint arXiv:1909.13021}, 2019.

\bibitem{shi1997normalized}
Jianbo Shi and Jitendra Malik.
\newblock Normalized cuts and image segmentation.
\newblock In {\em Proceedings of IEEE computer society conference on computer
  vision and pattern recognition}, pages 731--737. IEEE, 1997.

\bibitem{spielman2004nearly}
Daniel~A Spielman and Shang-Hua Teng.
\newblock Nearly-linear time algorithms for graph partitioning, graph
  sparsification, and solving linear systems.
\newblock In {\em Proceedings of the thirty-sixth annual ACM symposium on
  Theory of computing}, pages 81--90, 2004.

\bibitem{Tong2020DigraphIC}
Z.~Tong, Yuxuan Liang, Changsheng Sun, Xinke Li, David~S. Rosenblum, and
  A.~Lim.
\newblock Digraph inception convolutional networks.
\newblock In {\em NeurIPS}, 2020.

\bibitem{tong:directedGCN2020}
Zekun Tong, Yuxuan Liang, Changsheng Sun, David~S. Rosenblum, and Andrew Lim.
\newblock Directed graph convolutional network.
\newblock arXiv:2004.13970, 2020.

\bibitem{velivckovic2017graph}
Petar Veli{\v{c}}kovi{\'{c}}, Guillem Cucurull, Arantxa Casanova, Adriana
  Romero, Pietro Li{\`{o}}, and Yoshua Bengio.
\newblock {Graph Attention Networks}.
\newblock {\em International Conference on Learning Representations}, 2018.

\bibitem{Palmer2021}
Palmer W.R. and Zheng T.
\newblock Spectral clustering for directed networks.
\newblock {\em Studies in Computational Intelligence}, 943, 2021.

\bibitem{wu:gnnSurvey2020}
Zonghan Wu, Shirui Pan, Fengwen Chen, Guodong Long, Chengqi Zhang, and
  Philip~S. Yu.
\newblock A comprehensive survey on graph neural networks.
\newblock {\em IEEE Transactions on Neural Networks and Learning Systems},
  32(1):4--24, 2020.

\bibitem{xu2018powerful}
Keyulu Xu, Weihua Hu, Jure Leskovec, and Stefanie Jegelka.
\newblock How powerful are graph neural networks?
\newblock {\em arXiv preprint arXiv:1810.00826}, 2018.

\bibitem{zhou2018graph}
Jie Zhou, Ganqu Cui, Zhengyan Zhang, Cheng Yang, Zhiyuan Liu, Lifeng Wang,
  Changcheng Li, and Maosong Sun.
\newblock Graph neural networks: A review of methods and applications.
\newblock {\em arXiv preprint arXiv:1812.08434}, 2018.

\end{thebibliography}

\end{document}